\documentclass{article}

\usepackage{listings}
\usepackage{soul}
\usepackage{url}
\usepackage[utf8]{inputenc}
\usepackage{graphicx}
\usepackage{amsmath}
\usepackage{amsthm}
\usepackage{booktabs}
\usepackage{algorithm}
\usepackage{algorithmic}
\usepackage{adjustbox}

\usepackage{natbib}
\usepackage{xspace}
\usepackage{enumerate,paralist}
\usepackage{subcaption}

\usepackage{amsthm}
\usepackage{tikz}
\usetikzlibrary{arrows,positioning,shapes.misc,shapes.geometric}
\usetikzlibrary{arrows.meta}   

\usepackage{todonotes}
\usepackage{amsmath}
\usepackage{amssymb}
\newtheorem{example}{Example}

\newtheorem{definition}{Definition}

\newtheorem{corollary}{Corollary}

\newtheorem{remark}{Remark}

\newtheorem{assumption}{Important Assumption}
\usepackage{apxproof}
\newtheoremrep{propositionAp}{Proposition}
\newtheoremrep{factAp}{Fact}
\newtheoremrep{lemmaAp}{Lemma}

\newcommand {\tup}[1]      {{\langle #1 \rangle}}

\definecolor{dgreen}{rgb}{0.,0.6,0.}

 \newcommand{\jss}[1]{\textcolor{blue!20!black}{#1}}

\newcommand{\IC}{\mathit{IC}}
\newcommand{\ICc}{\mathcal{IC}}

\DeclareMathOperator{\lfp}{lfp}
	\newcommand{\Lc}{\mathcal{L}}

\newcommand{\indep}{\perp \!\!\! \perp}

\newcommand{\lpnot}{\mathrm{not}\:}
\renewcommand{\sim}{\lpnot }
\renewcommand{\lnot}{\lpnot}

\makeatletter
\providecommand{\bigsqcap}{%
  \mathop{%
    \mathpalette\@updown\bigsqcup
  }%
}
\newcommand*{\@updown}[2]{%
  \rotatebox[origin=c]{180}{$\m@th#1#2$}%
}
\makeatother
\DeclareRobustCommand\sampleline[1]{%
  \tikz\draw[#1, thick,line width=1] (0,0) (0,\the\dimexpr\fontdimen22\textfont2\relax)
  -- (2em,\the\dimexpr\fontdimen22\textfont2\relax);%
  \newcommand{\proofShifty}{\appendixproof}

}
\begin{document}

\title{An Algebraic Notion of Conditional Independence, and Its Application to Knowledge Representation (full version)}

\author{
Jesse Heyninck\\
{\small 
Open Universiteit, the Netherlands}\\
{\small University of Cape Town, South-Africa}
}
%

\maketitle

\setcounter{secnumdepth}{2}

\begin{abstract}
Conditional independence is a crucial concept supporting adequate modelling and efficient reasoning in probabilistics. In knowledge representation, the idea of conditional independence has also been introduced for specific formalisms, such as propositional logic and belief revision. In this paper, the notion of conditional independence is studied in the algebraic framework of approximation fixpoint theory. This gives a language-independent account of conditional independence that can be straightforwardly applied to any logic with fixpoint semantics. It is shown how this notion allows to reduce global reasoning to parallel instances of local reasoning, leading to fixed-parameter tractability results. Furthermore, relations to existing notions of conditional independence are discussed and the framework is applied to normal logic programming. 
\end{abstract}

\section{Introduction}

Over the last decades, conditional independence was shown to be a crucial concept supporting adequate modelling and  efficient reasoning in probabilistics \citep{pearl1989conditional}.  It is the fundamental concept underlying  network-based reasoning in probabilistics, which has been arguably one of the most important factors in the rise of contemporary artificial intelligence. Even though many reasoning tasks on the basis of probabilistic information 
have a high worst-case complexity due to their semantic nature, network-based models allow an efficient computation of many concrete instances of these reasoning tasks thanks to local reasoning techniques. Therefore, conditional independence has also been investigated for several approaches in knowledge representation, such as propositional logic \citep{darwiche1997logical,lang2002conditional}, belief revision \citep{kern2022conditional,lynn2022using} and conditional logics \citep{AAAI/HeyninckKM23}. For many other central formalisms in KR, such a study has not yet been undertaken.  

Due to the wide variety of formalisms studied in knowledge representation, it is often beneficial yet challenging to study a concept in a language-independent manner. Indeed, such language-independent studies avoid having to define and investigate the same concept for different formalisms. In recent years, a promising framework for such language-independent investigations is the algebraic \emph{approximation fixpoint theory} (AFT) \cite{denecker2003uniform}, which conceives of KR-formalisms as operators over a lattice (such as the immediate consequence operator from logic programming). AFT can represent a wide variety of KR-formalisms; such as propositional 
logic programming~\cite{denecker2012approximation,heyninck2024operatorbasedsemanticschoiceprograms,DBLP:journals/tplp/HeyninckB23}, default logic~\cite{denecker2003uniform}, autoepistemic logic~\cite{denecker2003uniform}, abstract argumentation and abstract 
dialectical frameworks~\cite{strass2013approximating}, hybrid MKNF~\cite{liu2022alternating} and SHACL~\cite{DBLP:journals/corr/abs-2109-08285},
 and was shown to be fruitful for language-independent studies of concepts such as splitting \citep{vennekens2006splitting}, 
or non-determinism \citep{DBLP:journals/corr/abs-2211-17262}.

In this paper, we give an algebraic, operator-based account of conditional independence. 
In more detail, the paper makes the following contributions:

\noindent (1) Definition of conditional independence in an operator-based, algebraic framework, providing a notion applicabble to any formalism that admits an operator-based characterization, such as the ones mentioned above.

\noindent (2) Proof of crucial properties about conditional independence, including that search for fixpoints of an (approximation) operator over conditionally indepedent modules.

\noindent (3) Fixed-parameter tractability results based on conditional independence.

\noindent (4) A proof-of-concept application to normal logic programs under various semantics.

\noindent (5) Establish connections with existing work on conditional independence in KR.

\noindent{\bf Outline of the Paper}:
The necessary preliminaries on logic programming (Section \ref{sec:LP}), lattices (Section \ref{sec:lattices}) and approximation fixpoint theory (Section \ref{sec:AFT}) are introduced in Section \ref{sec:back:prelim}. The concept of conditional independence of sub-lattices w.r.t.\ an operator is introduced and studied in Section \ref{sec:conditional:independence}, and applied to approximation operators in Section \ref{sec:conditional:independence:aft}. Fixed-parameter complexity results are shown in Section \ref{sec:fpt}.
This theory is applied to the semantics of normal logic programs  in Section \ref{sec:conditional:independence:lp}. Related work is discussed in Section \ref{sec:rel:work}, after which the paper is concluded (Section \ref{sec:conclusion}).

\section{Background and Preliminaries}
\label{sec:back:prelim}

In this section, we recall the necessary basics of logic programming,  abstract algebra, and AFT. 
We refer to the literature for more detailed introductions on logic programming semantics based on four-valued operators \citep{fitting1991bilattices,denecker2012approximation}, order theory \cite{davey2002introduction} and AFT \cite{phd/Bogaerts15}.

\subsection{Logic Programming}
\label{sec:LP}
We assume a set of atoms ${\cal A}$ and a language ${\cal L}$ built up from atoms, conjunction $\land$ and negation ${\rm not}$.
A (propositional normal) logic program ${\cal P}$ (a nlp, for short) is a finite set of rules of the form $p~\leftarrow~p_1\land \ldots \land p_n\land \lnot q_1\land \dots\land \lnot q_m$, 
where $p, p_1,\ldots,p_n,q_1,\ldots,q_m$ are atoms
that may include the propositional 
constants $\top$ (representing truth) and $\bot$ (falsity). 
A rule is {\em positive\/} if there are no negations in the rule's bodies, and a program is positive if so are all its rules. We use the  four-valued bilattice consisting of truth values ${\sf U}$, ${\sf F}$, ${\sf T}$ and ${\sf C}$ ordered by $\leq_i$ by: ${\sf U}\leq_i {\sf F} \leq_i {\sf C}$ and ${\sf U}\leq_i {\sf T} \leq_i {\sf C}$, and ordered by $\leq_t$ by: ${\sf F}\leq_t {\sf C}\leq_t {\sf T}$ and ${\sf F}\leq_t {\sf U}\leq_t {\sf T}$.

We also assume a $\leq_t$-involution $-$ on $\leq_t$ (i.e, $-{\sf F}={\sf T}$, $-{\sf T}={\sf F}$, $-{\sf U}={\sf U}$ and 
$-{\sf C}={\sf C}$). A {\em four-valued interpretation} of a program ${\cal P}$ is a pair $(x,y)$, where 
$x \subseteq {\cal A}_{\cal P}$ is the set of the atoms that are assigned a value in $\{{\sf T},{\sf C}\}$ and $y \subseteq {\cal A}_{\cal P}$ is the set of atoms assigned 
a value in $\{{\sf T},{\sf U}\}$. Somewhat skipping ahead to section \ref{sec:lattices}, the intuition here is that $x$ 
($y$) is a lower 
(upper) approximation of the true atoms.
Interpretations are compared by  the \emph{information order} $\leq_i$, where $(x,y)\leq_i (w,z)$ iff $x\subseteq w$ and $z\subseteq y$
(sometimes called ``precision'' order), and by the \emph{truth order} $\leq_t$, where $(x,y)\leq_t (w,z)$ iff $x\subseteq w$ and $y\subseteq z$
(increased `positive' evaluations). Truth assignments to complex formulas are then recursively defined as follows:
\begin{itemize}
\item $(x,y)(\jss{p})=
\begin{cases}
      {\sf T} & \text{ if } \jss{p} \in x \text{ and } \jss{p} \in y,  \\
      {\sf U} & \text{ if } \jss{p} \not\in x \text{ and }\jss{p} \in y,  \\
      {\sf F} & \text{ if } \jss{p} \not\in x \text{ and } \jss{p} \not\in y,  \\
      {\sf C} & \text{ if } \jss{p} \in x \text{ and } \jss{p} \not\in y. 
\end{cases}$ \smallskip
\item $(x,y)(\lnot \phi)=- (x,y)(\phi)$,  
\item $(x,y)(\psi \land \phi)=\bigsqcup_{\leq_t}\{(x,y)(\phi),(x,y)(\psi)\}$, 
\end{itemize}

A four-valued interpretation of the form $(x,x)$ may be associated with a {\em two-valued\/} (or {\em total\/}) interpretation $x$, 
in which for an atom $p$, $x(p) = {\sf T}$ if $p \in x$ and $x(p) = {\sf F}$ otherwise. We say that $(x,y)$ is a {\em three-valued\/} 
(or {\em consistent\/}) interpretation, if $x \subseteq y$. Note that in consistent interpretations there are no ${\sf C}$-assignments.

We now consider  the two- and four-valued immediate consequence operators for nlps,  defined as follows:

\begin{definition}
\label{def:operator:disj:lp}
Given a nlp ${\cal P}$ and a two-valued interpretation $x$, we define:
\[ IC_{\cal P}(x)=\{p\in {\cal A}_{\cal P}\mid p\leftarrow \psi\in {\cal P}, (x,x)(\psi)={\sf T} \}.\] 

For a four-valued interpretation $(x,y)$, we define:
\begin{align*}
 \ICc^l_{\cal P}(x,y)&=\{p\mid p\leftarrow \psi\in {\cal P}, (x,y)(\psi)\in\{{\sf T},{\sf C}\} \},\\
  \ICc^u_{\cal P}(x,y)&=\{p\mid p\leftarrow \psi\in {\cal P}, (x,y)(\psi)\in\{{\sf U},{\sf T}\} \},\\
\ICc_{\cal P}(x,y)&=( \ICc^l_{\cal P}(x,y),  \ICc^u_{\cal P}(x,y)).
\end{align*} 
 
 \end{definition}
 
Again somewhat skipping ahead to Section \ref{sec:lattices}, denoting by $2^{\cal A}$ the powerset of ${\cal A}$, $\IC_{\cal P}$ is an operator on the lattice $\tup{2^{\cal A},\subseteq}$ that derives all heads of rules with true bodies. $\ICc_{\cal P}$, on the other hand, is a generalisation of this operator to $(2^{\cal A})^2$.

\subsection{Lattices and sub-lattices}
\label{sec:lattices}

We recall some necessary preliminaries on set theory and (sub-)lattices.
A lattice is a partially ordered set $L=\langle {\cal L},\leq\rangle$ where every two elements $x,y\in{\cal L}$ have a least upper $x\sqcup y$ and a greatest lower bound $x\sqcap y$. We will often also refer to a lattice $\langle {\cal L},\leq\rangle$ by its set of elements ${\cal L}$. A lattice is complete if every set $X\subseteq {\cal L}$ has a least upper (denoted $\bigsqcup X$) and a greatest lower bound (denoted $\bigsqcap X$). $\tup{2^{\cal A},\subseteq}$  is an example of a complete lattice. $x$ is a fixpoint of $O$ if $x=O(x)$, and the least fixpoint of $O$ is denoted $\lfp(O)$.

We now provide background on how to (de)compose lattices into sub-lattices, following \citet{vennekens2006splitting}.
Let $I$ be a set, which we call the \emph{index set}, and for each $i\in I$, let $\Lc_i$ be a set. The product set $\bigotimes_{i\in I}\Lc_i$ is the following set of functions:
\[
\bigotimes_{i\in I} \Lc_i= \{f\mid f:I\rightarrow \bigcup_{i\in I} \Lc_i\mbox{ s.t.\ } \forall i\in I: \: f(i)\in \Lc_i\}.\]
The product set $\bigotimes_{i\in I} \Lc_i$ contains all ways of selecting one element of every set $\Lc_i$. 
E.g.\ for the sets $\Lc_1=\{\emptyset,\{p\}\}$ and $\Lc_2=\{\emptyset,\{q\}\}$, $\bigotimes_{i\in \{1,2\}}\Lc_i$ contains, among others, $f$ and $f'$ with $f(1)=f(2)=\emptyset$ and $f'(1)=\emptyset$ and $f'(2)=\{q\}$.
For finite $I=\{1,\ldots,n\}$, the product $\bigotimes_{i\in I}\Lc_i$ is (isomorphic to) the cartesian product $\Lc_1\times\ldots \times \Lc_n$. We will also denote $\bigotimes_{i\in \{1,2\}}\Lc_i$ by $\Lc_1\otimes \Lc_2$ to avoid clutter.

If each $\Lc_j$ is partially ordered by some $\leq_j$, this induces the product order $\leq_\otimes$ on $\bigotimes_{j\in I}\Lc_j$: for all $x,y\in \bigotimes_{j\in I}\Lc_j$, $x\leq_\otimes y$ iff for all $j\in I$, $x(j)\leq_j y(j)$. We will sometimes denote the product order over $\bigotimes_{j\in I}\Lc_j$ by $\leq_\otimes^I$. It can be easily shown that if all $\langle \Lc_j,\leq_j\rangle$ are (complete) lattices, then $\langle \bigotimes_{j\in I}\Lc_j,\leq_\otimes\rangle$ is also a (complete) lattice, called  the \emph{product lattice} of the lattices $\Lc_j$.

We denote, for $x\in \bigotimes_{i\in I} \Lc_i$ and $i\in I$, $x_{|i}\in \Lc_i$ as $f(i)$, and for $J\subseteq I$ we denote $x_{|J}$ by $\bigotimes_{i\in J}x_i$. For example, using $\Lc_1$ and $\Lc_2$ as in the example above, $\emptyset\times \{q\}_{|1}=\emptyset$.
Likewise, we denote by $x_i\otimes x_j$ the element $x\in \Lc_i\otimes \Lc_j$ s.t.\ $x_{|k}=x_k$ for $k=i,j$, and we lift this to sets as usual.

\subsection{Approximation Fixpoint Theory}
\label{sec:AFT}
We recall basic notions from approximation fixpoint theory (AFT) by  \cite{denecker2000approximations}.  

 Given a lattice $L = \tup{{\cal L},\leq}$, we let $L^2 =\tup{{\cal L}^2,\leq_i,\leq_t}$ be the structure (called {\em bilattice\/}), in which
${\cal L}^2 = {\cal L} \times {\cal L}$, and for every $x_1,y_1,x_2,y_2 \in {\cal L }$, \smallskip \\
$\bullet$ $(x_1,y_1) \leq_i (x_2,y_2)$ if $x_1 \leq x_2$ and $y_1 \geq y_2$, \smallskip \\
$\bullet$ $(x_1,y_1) \leq_t (x_2,y_2)$ if $x_1 \leq x_2$ and $y_1 \leq y_2$.

An \emph{approximating operator} ${\cal O}:{\cal L}^2\rightarrow {\cal L}^2$ of an operator $O:{\cal L}\rightarrow {\cal L}$
is an operator that maps every approximation $(x,y)$ of an element $z$ to an approximation $(x',y')$ of another element $O(z)$, thus approximating the behavior of the approximated operator $O$. As an example, $\ICc_{\cal P}$ is an approximation operator of $\IC_{\cal P}$.

\begin{definition}
\label{def:approx-notions}
Let $O_{\cal L}:{\cal L}\rightarrow {\cal L}$ and ${\cal O}:{\cal L}^2\rightarrow {\cal L}^2$. (1)           ${\cal O}$ is {\em $\leq_i$-monotonic\/}, if when $(x_1,y_1)\leq_i(x_2,y_2)$, also ${\cal O}(x_1,y_1)\leq_i {\cal O}(x_2,y_2)$; 
 (2)            ${\cal O}$ is \emph{approximating\/}, if it is $\leq_i$-monotonic and for any $x\in {\cal L}$, $({\cal O}(x,x))_1=({\cal O}(x,x))_2$;\footnote{In some papers \citep{denecker2000approximations}, an approximation operator is defined as a symmetric $\leq_i$-monotonic operator, i.e.\ a $\leq_i$-monotonic operator s.t.\ for every $x,y\in {\cal L}$, ${\cal O}(x,y)=({\cal O}_l(x,y),{\cal O}_l(y,x))$ for some ${\cal O}_l:{\cal L}^2\rightarrow {\cal L}$. However, the weaker condition we take here (taken from \citet{denecker2002ultimate}) is actually sufficient for most results.} 
(3) ${\cal O}$ is an {\em approximation\/} of $O_{\cal L}$, if it is $\leq_i$-monotonic and
             ${\cal O}$ \emph{extends} $O$, that is: $({\cal O}(x,x))_1=({\cal O}(x,x))_2=O_{\cal L}(x)$.
 \end{definition}
To avoid clutter, we denote $({\cal O}(x,y))_1$ by ${\cal O}_l(x,y)$ and $({\cal O}(x,y))_2$ by ${\cal O}_u(x,y)$ (for \textit{lower} and \textit{upper} bound).

The \emph{stable operator\/}, defined next, is used for expressing the semantics of many non-monotonic formalisms.
\begin{definition}
 Given a complete lattice $L = \tup{{\cal L},\leq}$,  let ${\cal O}:{\cal L}^2 \rightarrow {\cal L}^2$ be an approximating operator.
${\cal O}_{l}(\cdot,y) = \lambda x.{\cal O}_{l}(x,y)$, i.e.: ${\cal O}_{l}(\cdot,y)(x) = {\cal O}_{l}(x,y)$ (and similarly for the upper bound operator ${\cal O}_u$).
 The \emph{stable operator for ${\cal O}$\/} is: 
         $S({\cal O})(x,y)=(\mathit{lfp}({\cal O}_l(.,y)),\mathit{lfp}({\cal O}_u(x,.))$.
         \end{definition}
The components $\mathit{lfp}({\cal O}_l(.,y))$ and $\mathit{lfp}({\cal O}_u(x,.)$ of $S({\cal O})$ will be denoted by $C({\cal O}_l)(y)$ respectively $C({\cal O}_u)(x)$.

Stable operators capture the idea of minimizing truth, since for any $\leq_i$-monotonic operator ${\cal O}$ on ${\cal L}^2$,
fixpoints of the stable operator $S({\cal O})$ are $\leq_t$-minimal fixpoints of ${\cal O}$ \cite[Theorem~4]{denecker2000approximations}.
Altogether, we obtain the following notions:
 Given a complete lattice $L = \tup{{\cal L},\leq}$,  let ${\cal O}:{\cal L}^2 \rightarrow {\cal L}^2$ be an approximating operator, $(x,y)$ is
 \begin{itemize}
 \item a \emph{Kripke-Kleene fixpoint} of ${\cal O}$ if $(x,y) =\lfp_{\leq_i}({\cal O}(x,y))$;
\item a \emph{three-valued stable} fixpoint of ${\cal O}$ if $(x,y)= S({\cal O})(x,y) $;
\item  a \emph{two-valued stable fixpoints} of ${\cal O}$ if $(x,y)= S({\cal O})(x,y)$; and $x=y$;
\item  the \emph{well-founded  fixpoint} of ${\cal O}$ if it is the $\leq_i$-minimal (three-valued) stable model fixpoint of ${\cal O}$. 
\end{itemize}
It has been shown that every approximation operator admits a unique $\leq_i$-minimal stable fixpoint \citep{denecker2000approximations}.
 \citet{pelov2007well} show that for normal logic programs, the fixpoints based on the four-valued immediate consequence operator $\ICc_{\cal P}$ (recall Definition \ref{def:operator:disj:lp})
for a logic program ${\cal P}$ give rise 
to the following correspondences: the three-valued stable fixpoints of $\ICc_{\cal P}$ coincide with the three-valued  stable semantics as defined by \citet{przymusinski1990well}, the well-founded 
fixpoint of  of $\ICc_{\cal P}$  coincides with the homonymous semantics \citep{przymusinski1990well,van1991well}, and the two-valued stable fixpoints of  $\ICc_{\cal P}$  coincide with the two-valued 
(or total) stable models.

\section{Conditional Independence}
\label{sec:conditional:independence}
Conditional independence in an operator-based setting is meant to formalize the idea that for the application of an operator to a lattice consisting of three sub-lattices ${\cal L}_1$, ${\cal L}_2$ and ${\cal L}_3$, full information about ${\cal L}_3$ allows us to ignore ${\cal L}_2$ when applying $O$ to ${\cal L}_1\otimes {\cal L}_3$.
In more detail, it means that the operator $O$ over $\bigotimes_{i\in \{1,2,3\}}\Lc_i$ can be decomposed in two operators $O_{1,3}$ and $O_{2,3}$ over the sub-lattices ${\cal L}_1\otimes {\cal L}_3$ respectively ${\cal L}_2\otimes {\cal L}_3$.
\begin{definition}
Let $O$ be an operator $\bigotimes_{i\in \{1,2,3\}} \Lc_i$. The lattices \emph{$\Lc_1$ and $\Lc_2$ are independent w.r.t.\ $\Lc_3$ according to $O$}, in symbols: $\Lc_1\indep_{O}\Lc_2\mid \Lc_3$, if there exist operators 
 \begin{eqnarray*}
 O_{i,3}:{\cal L}_i\otimes {\cal L}_3\rightarrow {\cal L}_i\otimes {\cal L}_3 \quad\mbox{ for }\:i=1,2
  \end{eqnarray*}
s.t.\ for every $x_1\in {\cal L}_1, x_2\in {\cal L}_2, x_3\in  {\cal L}_3$ it holds that:
 \begin{eqnarray*}
 &&O(x_1\otimes x_2\otimes x_3)_{|1,3}=O_{1,3}(x_1\otimes x_3)\quad \mbox{ and }\\
 &&O(x_1\otimes x_2\otimes x_3)_{|2,3}=O_{2,3}(x_2\otimes x_3)
 \end{eqnarray*}
\end{definition}
 Thus, two sub-lattices ${\cal L}_1$ and ${\cal L}_2$ are independent w.r.t.\ ${\cal L}_3$ according to $O$ if, once we have full information about ${\cal L}_3$, information about ${\cal L}_2$ does not contribute anything in the application of $O$ when restricted to ${\cal L}_1$ (and vice versa). Where $\Lc_1\indep_{O}\Lc_2\mid \Lc_3$, we will also call $\Lc_3$ the \emph{conditional pivot}, and will refer to members of $\Lc_3$ as such as well.

\begin{example}\label{example:infection:1}
Consider the logic program ${\cal P}$ using atoms for $\mathtt{inf}$ected, $\mathtt{vac}$cinated and $\mathtt{c}$o$\mathtt{n}$ta$\mathtt{ct}$:
\begin{eqnarray*}
&&r_1:\ \mathtt{inf}(\mathtt{b})\leftarrow \mathtt{inf}(\mathtt{a}),\mathtt{cnct}(\mathtt{a},\mathtt{b}), \mathrm{not}\: \mathtt{vac}(\mathtt{b}).\\
&&r_2:\ \mathtt{inf}(\mathtt{c})\leftarrow \mathtt{inf}(\mathtt{a}),\mathtt{cnct}(\mathtt{a},\mathtt{c}), \mathrm{not}\: \mathtt{vac}(\mathtt{c}).\\
&&r_3:\ \mathtt{inf}(\mathtt{a}).,\: r_4:\: \mathtt{cnct}(\mathtt{a},\mathtt{b}).,\: r_5:\: \mathtt{cnct}(\mathtt{a},\mathtt{c}).
\end{eqnarray*} 
Notice that, as soon as we know that $\mathtt{inf}(\mathtt{a}).$ is the case, we can decompose the search for models into  two independent parts, as can also be seen in the dependency graph in figure \ref{fig:dependency:graph}.
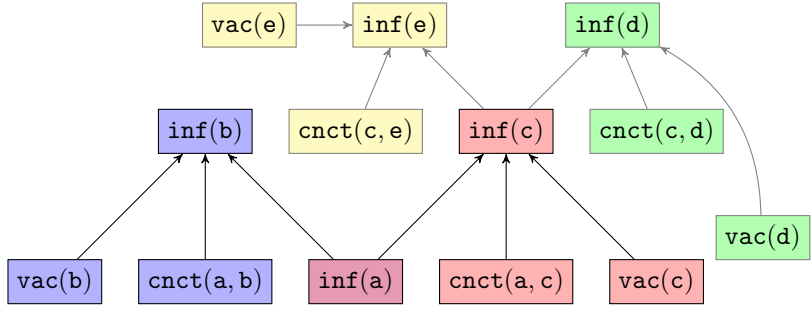
\begin{figure}[h]
\centering
{
\begin{tikzpicture}[->, >=stealth', node distance=2cm, scale=0.5]
    \node[draw,rectangle, fill=purple!40!white] (IJ) {$\mathtt{inf}(\mathtt{a})$};
    \node[draw,rectangle, left of=IJ, fill=blue!30!white] (CI) {$\mathtt{cnct}(\mathtt{a},\mathtt{b})$};
    \node[draw,rectangle, right of=IJ, fill=red!30!white] (CT) {$\mathtt{cnct}(\mathtt{a},\mathtt{c})$};
    \node[draw,rectangle, above of=CI, fill=blue!30!white] (II) {$\mathtt{inf}(\mathtt{b})$};
    \node[draw,rectangle, above of=CT, fill=red!30!white] (IT) {$\mathtt{inf}(\mathtt{c})$};
    \node[draw,rectangle, left of=CI, fill=blue!30!white] (VI) {$\mathtt{vac}(\mathtt{b})$};
    \node[draw,rectangle, right of=CT, fill=red!30!white] (VT) {$\mathtt{vac}(\mathtt{c})$};

    \node[draw=gray,rectangle,  above right of=IT, fill=green!30!white] (ID) {$\mathtt{inf}(\mathtt{d})$};
    \node[draw=gray,rectangle,  right of=IT, fill=green!30!white] (CD) {$\mathtt{cnct}(\mathtt{c},\mathtt{d})$};    
        \node[draw=gray,rectangle,  below right of=CD, fill=green!30!white] (VD) {$\mathtt{vac}(\mathtt{d})$};    

        \node[draw=gray,rectangle,  above left of=IT, fill=yellow!30!white] (IE) {$\mathtt{inf}(\mathtt{e})$};
        \node[draw=gray,rectangle,  left of=IT, fill=yellow!30!white] (CE) {$\mathtt{cnct}(\mathtt{c},\mathtt{e})$};   
            \node[draw=gray,rectangle,  left of=IE, fill=yellow!30!white] (VE) {$\mathtt{vac}(\mathtt{e})$};   

    \path (IJ) edge (II);
    \path (IJ) edge (IT);
    \path (CI) edge (II);
    \path (CT) edge (IT);
    \path (VI) edge (II);
    \path (VT) edge (IT);
    
     \path[color=gray] (IT) edge (IE);
     \path[color=gray] (IT) edge (ID);
     \path[color=gray] (CE) edge (IE);
     \path[color=gray] (VE) edge (IE);
	 \path[color=gray] (CD) edge (ID);
     \path[color=gray] (VD) edge[bend right=30] (ID);
\end{tikzpicture}}

\caption{\label{fig:dependency:graph} A dependency graph for the program ${\cal P}_1$ (Example 1), and its extension ${\cal P}_2$ (Example \ref{example:cit:tree}, atoms only occuring in ${\cal P}_2$ have gray outlines.).}
\end{figure}

As a product lattice consisting of power sets of sets ${\cal A}_1,\ldots,{\cal A}_3$ is isomorphic to the powerset of the union of these sets ${\cal A}_1\cup\ldots\cup{\cal A}_3$, we shall use them interchangeably. 
We let:\begin{eqnarray*}
{\cal A}_b & = & \{\mathtt{inf}(\mathtt{b}),\mathtt{cnct}(\mathtt{a},\mathtt{b}),\mathtt{vac}(\mathtt{b})\}\\
{\cal A}_c & = & \{\mathtt{inf}(\mathtt{c}),\mathtt{cnct}(\mathtt{a},\mathtt{c}),\mathtt{vac}(\mathtt{c})\}\\
{\cal A}_a & = & \{\mathtt{inf}(\mathtt{a})\}
\end{eqnarray*}
We see that $2^{{\cal A}_b}\indep_{\IC_{{\cal P}}} 2^{{\cal A}_c}\mid 2^{{\cal A}_a}$, by observing that:
\begin{eqnarray*}
\IC_{{\cal P}}^{{\cal A}_b,{\cal A}_a}=\IC_{{\cal P}^{{\cal A}_b,{\cal A}_a}} \mbox{ and } 
\IC_{{\cal P}}^{{\cal A}_c,{\cal A}_a}=\IC_{{\cal P}^{{\cal A}_c,{\cal A}_a}}
\end{eqnarray*}
where 
${\cal P}^{{\cal A}_b,{\cal A}_a}=  \{r_1,r_3,r_4\}$ and ${\cal P}^{{\cal A}_c,{\cal A}_a}=\{r_2,r_3,r_5\}$.
It is easily verified that for every $x_1\cup x_2\cup x_3\subseteq {\cal A}_{\cal P}$, it holds that $\IC_{\cal P}(x_b\cup x_c\cup x_a)\cap ({\cal A}_i\cup{\cal A}_a)=\IC_{{\cal P}^{{\cal A}_i,{\cal A}_a}}(x_i\cup x_a)$ for any $i=b,c$.
\end{example} 
 
Intuitively, our notion of conditional independence relates to the analogous notion known from probability theory as follows:
 two events $x_1$ and $x_2$ in probability theory are independent conditional on a third $x_3$ if we can calculate the probability of $x_1$ given $x_2$ and $x_3$ as the probability of $x_1$ given $x_3$, i.e.\ $P(x_1\mid x_2 x_3)=P(x_1\mid x_3)$. The idea of conditional probability relative to an operator is the same: we can apply the operator to one of the two sub-lattices independently of the second, independent sub-lattice.

We now undertake a study of the properties of operators that respect conditional independencies.  In probability theory, an equivalent definition of conditional independence is given by: $P(x_1,x_2|x_3)=P(x_1|x_3)P(x_2|x_3)$. This property or definition is mirrored in the following fact, where we show that the application of an operator over the entire lattice can likewise be split up over the two independent sublattices $\Lc_1$ and $\Lc_2$ (together with the conditional pivot $\Lc_3$):

\begin{factAprep}
Let an operator $O$  on $\otimes_{i\in \{1,2,3\}} {\cal L}_i$ s.t.\ $\Lc_1\indep_O \Lc_2\mid \Lc_3$ be given. Then  for any $x_1\otimes x_2\otimes x_3\in \bigotimes_{i\in \{1,2,3\}}{\cal L}_i$, it holds that:
\begin{align*}
O(x_1\otimes x_2\otimes x_3) &= O_{1,3}(x_1\otimes x_3)\otimes O_{2,3}(x_2\otimes x_3)_{|2}\\
&=O_{1,3}(x_1\otimes x_3)_{|1}\otimes O_{2,3}(x_2\otimes x_3) .
\end{align*}
Furthermore, for any $i,j=1,2$, $i\neq j$, $x_i\in{\cal L}_i, x_j,x'_j\in {\cal L}_j$ and $x_3\in{\cal L}_3$ it holds that:
\begin{align*}
O(x_i\otimes x_j\otimes x_3)_{|i,3}=O(x_i\otimes x'_j\otimes x_3)_{|i,3}
\end{align*}
\end{factAprep} 
\begin{appendixproof}
Immediate in view of the definition of conditional independence.
\end{appendixproof}

In the appendix, we also study the semi-graphoid-properties \citep{pearl1989conditional}.

\begin{toappendix}
\begin{remark}
We now demonstrate that our conditional independence also shows some differences with conditional independence as known from probability theory. For example, not all semi-graphoid-properties \citep{pearl1989conditional} are satisfied. In more detail, whereas \emph{symmetry} (i.e.\ $\Lc_1\indep_O \Lc_2\mid \Lc_3$ implies $\Lc_2\indep_O \Lc_1\mid \Lc_3$) is obviously satisfied, the properties of \emph{decomposition} (i.e.\  $\Lc_1\indep_O \Lc_2\otimes \Lc_3\mid \emptyset$ implies $\Lc_1\indep_O\Lc_2\mid \emptyset$) and \emph{weak union} (i.e.\  $\Lc_1\indep_O \Lc_2\otimes \Lc_3\mid \emptyset$ implies $\Lc_1\indep_O \Lc_2\mid \Lc_3$) are not satisfied. Regarding decomposition, it should be noted that this property is undefined as we assume conditional independence over decompositions of the complete lattice. A violation of weak union is illustrated in the following example:
\begin{example}
Consider the logic program ${\cal P}=\{a\leftarrow; b\leftarrow \lnot c; c\leftarrow \lnot b\}$. Note that $2^{\{a\}}\indep_{\IC_{\cal P}} 2^{\{b,c\}}\mid \emptyset$. Yet it does \emph{not} hold that $2^{\{a\}}\indep_{\IC_{\cal P}} 2^{\{c\}}\mid 2^{\{b\}}$, as:
\begin{eqnarray*}
&\IC_{\cal P}(\{a\})\cap \{a,c\}&=\{a,c\} \quad \neq \\
&\IC_{\cal P}(\{a,b\})\cap \{a,c\}&=\{a\}
\end{eqnarray*}
\end{example}
The reason for the failure of weak union is that we are not only interested in the behaviour of the operator $O$ w.r.t.\ the conditionally independent sub-lattices $L_1$ and $L_2$, but also take into account the conditional pivot $L_3$. Contrast this with probabilistic conditional independence where the defining condition $p(x_1\mid x_3)=p(x_1\mid x_2,x_3)$ only talks about $L_1$. The reason that conditional pivots are taken into account in our account of conditional indepdence is that we are usually interested in fixpoints of an operator. It might be interesting to look at a weaker notion of conditional independence w.r.t.\ operators that does not consider the conditional pivot in the output of the operator (and indeed, it is not hard to see that weak union is satisfied for such a notion), but due to our focus on fixpoints, we restrict attention to the stronger notion.
\end{remark}
\end{toappendix}
The output of an operator w.r.t.\ the conditional pivot only depends on the input w.r.t.\ the conditional pivot:
\begin{lemmaAprep}\label{lemma:operators:agree:on:pivot}
Let an operator $O$  on $\otimes_{i\in \{1,2,3\}} {\cal L}_i$ s.t.\ $\Lc_1\indep_O\Lc_2\mid\Lc_3$ be given. Then $O_{2,3}(x_2\otimes x_3)_{|3}=O_{1,3}(x_1\otimes x_3)_{|3}$ for every $x_1\otimes x_2\otimes x_3\in \otimes_{i\in \{1,2,3\}} {\cal L}_i$.
\end{lemmaAprep}
\begin{appendixproof}
As $\Lc_1\indep_O \Lc_2\mid \Lc_3$, for any  $x_1\otimes x_2\otimes x_3\in \otimes_{i\in \{1,2,3\}} {\cal L}_i$, $O(x_1\otimes x_2\otimes x_3)=O_{1,3}(x_1\otimes x_3)\otimes O_{2,3}(x_2\otimes x_3)_{|2}=O_{1,3}(x_1\otimes x_3)_{|1}\otimes O_{2,3}(x_2\otimes x_3)$, which implies $O_{1,3}(x_1\otimes x_3)_{|3}=O_{2,3}(x_2\otimes x_3)_{|3}$.
\end{appendixproof}

We now show one of the central results, namely that fixpoints of an operator $O$ respecting independence of $\Lc_1$ and $\Lc_2$ w.r.t.\ $\Lc_3$ can be obtained by combining the fixpoints of $O_{1,3}$ and $O_{2,3}$. Thus, the search for fixpoints, a central task in KR, can be split into two parallel, smaller searches:
\begin{propositionAprep}\label{prop:fixed:points:preserved}
Let an operator $O$  on  $\bigotimes_{i\in \{1,2,3\}} \Lc_i$ s.t.\ $\Lc_1\indep_O \Lc_2\mid \Lc_3$ be given. Then
$x_1\otimes x_2\otimes x_3=O(x_1\otimes x_2\otimes x_3)$ iff $x_1\otimes x_3=O_{1,3}(x_1\otimes x_3)$ and $x_2\otimes x_3=O_{2,3}(x_2\otimes x_3)$ (for any $x_1\otimes x_2\otimes x_3\in \bigotimes_{i\in \{1,2,3\}}\Lc_i$).
\end{propositionAprep}
\begin{appendixproof}
For the $\Rightarrow$-direction, suppose that $x_1\otimes x_2\otimes x_3=O(x_1\otimes x_2\otimes x_3)$. Since $\Lc_1\indep_O \Lc_2\mid\Lc_3$, $O_{i,3}(x_i\otimes x_3)=O(x)_{|i,3}=x_{|i,3}=x_i\otimes x_3$ (for $i=1,2$). For the $\Leftarrow$-direction, suppose that $x_1\otimes x_3=O(x_1\otimes x_3)$ and $x_2\otimes x_3=O(x_2\otimes x_3)$. As  $\Lc_1\indep_O \Lc_2\mid \Lc_3$, $O(x_1\otimes x_2\otimes x_3)=O_{1,3}(x_1\otimes x_3)\otimes O_{2,3}(x_2\otimes x_3)_{|2}=x_1\otimes x_2\otimes x_3=x$.
\end{appendixproof}

A second central insight is that monotonicity is preserved when moving between a product lattice and its components:
\begin{propositionAprep}\label{prop:monotonicity:preserved}
Let an operator $O$  on $\bigotimes_{i\in \{1,2,3\}} \Lc_i$ s.t.\ ${\Lc_1\indep_O \Lc_2\mid \Lc_3}$ be given. Then $O$ is $\leq_\otimes$-monotonic iff $O_{i,3}$ over ${\cal L}_i\otimes {\cal L}_3$ is $\leq^{i,3}_\otimes$-monotonic for $i=1,2$.
\end{propositionAprep}
\begin{appendixproof}
In what follows we let $I=\{1,2,3\}$.
For the $\Rightarrow$-direction, suppose that $O$ is $\leq^I_\otimes$-monotonic and consider some $x_1^1\otimes x_3^1\leq^{1,3}_\otimes x_1^2 \otimes x_3^2$. Notice that $O(x_1^1\otimes x_2 \otimes x_3^1)\leq_\otimes O(x_1^2\otimes x_2 \otimes x_3^2)$ for any $x_2 \in L_2$ (as $O$ is $\leq_\otimes^I$-monotonic). This means that $O_{1,3}(x_1^1\otimes  x_3^1)\leq^{1,3}_\otimes O_{1,3}(x_1^2 \otimes x_3^2)$ by definition of $\leq_\otimes$ and since $L_1\indep_O L_2\mid L_3$.
For the $\Leftarrow$-direction, suppose that $O_{i,3}$ are $\leq$-monotonic for $i=1,2$. Consider some $x^1,x^2\in \bigotimes_{i\in \{1,2,3\}} S_i$ with $x^1\leq^{I}_\otimes x^2$. Then $O_{i,3}(x^1_{|i,3})\leq^{i,3}_\otimes O_{i,3}(x^2_{|i,3})$ for $i=1,2$ which implies $O_{1,3}(x^1_{|1,3})\otimes O_{2,3}(x^1_{|2,3})_{|2}\leq^I_\otimes O_{1,3}(x^2_{|1,3})\otimes O_{2,3}(x^2_{|2,3})_{|2}$ by definition of $\leq^I_\otimes$. With conditional independence, $O(x^j)=O_{1,3}(x^j_{|1,3})\otimes O_{2,3}(x^j_{|2,3})_{|2}$ for $j=1,2$. Thus, $O(x^1)\leq^I_\otimes O(x^2)$.
\end{appendixproof}

Finally, for monotonic operators over complete lattices, the \emph{least} fixed points can be obtained by combining the least fixed points of conditionally independent sub-lattices:
\begin{propositionAprep}\label{prop:lfps:preserved}
Let a $\leq_\otimes$-monotonic operator $O$  on the complete lattice $\bigotimes_{i\in \{1,2,3\}} \Lc_i$ s.t.\ ${\Lc_1\indep_O \Lc_2\mid \Lc_3}$ be given. Then $x_1\otimes x_2\otimes x_3$ is the least fixed point of $O$ iff $x_i\otimes x_3$ is the least fixed point of $O_{i,3}$ (for $i=1,2$ and any $x_1\otimes x_2\otimes x_3\in\bigotimes_{i\in \{1,2,3\}} \Lc_i$).
\end{propositionAprep}
\begin{appendixproof}
Suppose first that $x_1\otimes x_2\otimes x_3$ is the least fixed point of $O$. We show that $x_1\otimes x_3$ is a least fixed point of $O_{1,3}$ (which suffices with symmetry).  With Proposition \ref{prop:fixed:points:preserved}, $x_1\otimes x_3=O_{1,3}(x_1\otimes x_3)$. 
It thus suffices to show that for any fixed point $x'_1\otimes x'_3$ of $O_{1,3}$, $x'_1\otimes x'_3\geq_\otimes^{I_1\cup I_3} x_1\otimes x_3$.  Assume thus that $x'_1\otimes x'_3=O_{1,3}(x'_1\otimes x'_3)$.
First, observe that $\bot_2\otimes x_3\leq_\otimes^{I_2\cup I_3} x'_2\otimes x_3$ for any $x'_2$. By Lemma \ref{lemma:operators:agree:on:pivot}, $O_{2,3}(x'_2\otimes x'_3)_{|3}=O_{1,3}(x'_1\otimes x'_3)=x'_3$ for any $x'_2\in {\cal L}_2$. Thus, $O_{2,3}(\bot_2\otimes  x'_3)\geq_\otimes^{I_2\cup I_3} \bot_2\otimes x'_3$, i.e.\ $\bot_2\otimes x'_3$ is a pre-fixpoint of $O_{2,3}$. We can apply $O_{2,3}$ inductively, and, as it is a $\leq_\otimes^{I_2\cup I_3}$-monotonic operator (Proposition \ref{prop:monotonicity:preserved}), a fixpoint is guaranteed to exist. Thus, there is some $x'_2\in{\cal L}_2$ s.t.\ $O(x'_2\otimes x'_3)=x'_2\otimes x'_3$. This means that $x'_1\otimes x'_2\otimes x'_3$ is a fixpoint of $O$, which implies $x_1\otimes x_2\otimes x_3\leq_\otimes x'_1\otimes x'_2\otimes x'_3$, which on its turn implies (by definition of $\leq_\otimes$), $x_1\otimes x_3\leq x'_1\otimes x'_3$.

Suppose now that $x_{|i,3}$ is a least fixed point of $O_{i,3}$ (for $i=1,2$). With Proposition \ref{prop:fixed:points:preserved}, $x_1\otimes x_2\otimes x_3$ is a fixed point of $O$. We show that for any fixed point $x'$ of $O$, $x'\geq_\otimes x_1\otimes x_2\otimes x_3$. Indeed, with Proposition \ref{prop:fixed:points:preserved}, $x'_{|i,3}$ is a fixed point of $O_{i,3}$, which implies that $x'_{|i,3}\geq x_{|i,3}$. By definition of $\leq_\otimes$, $x'_{1}\otimes x'_2\otimes x'_3\geq_\otimes x_1\otimes x_2\otimes x_3$.
\end{appendixproof}

\section{Conditional Independence and AFT}
\label{sec:conditional:independence:aft}
The notion of conditional independence is immediately applicable to approximation operators. 
In this section, we derive results on the modularisation of AFT-based semantics based on the results from the previous section. 

As observed by \citet{vennekens2006splitting}, the bilattice ${\cal L}^2$ of a product lattice ${\cal L}=\bigotimes_{i\in I} {\cal L}_i$ is isomorphic to the product lattice of bilattices $\bigotimes_{i\in I} {\cal L}^2_i$, and we move between these two constructs without further ado. 

As an approximation operator is a $\leq_i$-monotonic operator, we immediately obtain that the search for fixpoints, including the Kripke-Kleene fixpoints, can be split on the basis of conditional independence of an approximator:
\begin{propositionAprep}
Let an approximation operator ${\cal O}$ over a bilattice of the product lattice $\bigotimes_{i\in \{1,2,3\}} {\cal L}_i$ be given s.t.\ ${\cal L}^2_1\indep_{\cal O} {\cal L}^2_2\mid {\cal L}^2_3$. Then the following hold:
\begin{itemize}
\item  $(x,y)$ is the Kripke-Kleene fixpoint of ${\cal O}$ iff $(x_{|i,3},y_{|i,3})$ is the Kripke-Kleene fixpoint of ${\cal O}_{i,3}$ for $i=1,2$. 
\item $(x,y)$ is a fixpoint of ${\cal O}$ iff $(x_{|i,3},y_{|i,3})$ is a fixpoint of ${\cal O}_{i,3}$ for $i=1,2$. 
\end{itemize}
\end{propositionAprep}
\begin{appendixproof}
This is an immediate consequence of Propositions \ref{prop:fixed:points:preserved}, \ref{prop:monotonicity:preserved} and \ref{prop:lfps:preserved}.
\end{appendixproof}

We now turn our attention to the stable operators. We first investigate the relation between an approximation operator and the lower and upper-bound component of this operator when it comes to respecting independencies. It turns out that the component operators ${\cal O}_l$ and ${\cal O}_u$ respect conditional independencies, and, vice-versa, that the respect of the two component operators of conditional independencies implies respect of these independencies by the approximator:
\begin{propositionAprep}\label{prop:conditional:independence:approx:components}
Let an approximation operator ${\cal O}$ over a bilattice of the product lattice $\bigotimes_{i\in \{1,2,3\}} {\cal L}_i$ be given. Then ${\cal L}^2_1\indep_{{\cal O}} {\cal L}^2_2\mid {\cal L}^2_3$ iff ${\cal L}^2_1\indep_{{\cal O}_l} {\cal L}^2_2\mid {\cal L}^2_3$ and ${\cal L}^2_1\indep_{{\cal O}_u} {\cal L}^2_2\mid {\cal L}^2_3$.\footnote{Notice that we slightly abuse notation here, as ${\cal O}_l$ and ${\cal O}_u$ map from ${\cal L}^2$ to ${\cal L}$. However, one can easily obtain an operator ${\cal O}'_l$ on ${\cal L}^2$ by e.g.\ defining ${\cal O}_l(x,y)=({\cal O}'_l(x,y),\top)$.}
\end{propositionAprep}
\begin{appendixproof}
For the $\Rightarrow$-direction, suppose that ${\cal L}^2_1\indep_{{\cal O}} {\cal L}^2_2\mid {\cal L}^2_3$, i.e.\ there are some ${\cal O}_{i,3}:{\cal L}_i\otimes {\cal L}_3\rightarrow {\cal L}_i\otimes {\cal L}_3$ s.t.\
${\cal O}( (x_1,y_1)\otimes (x_2,y_2)\otimes (x_3,y_3))_{|i,3}=
{\cal O}_{i,3}((x_i,y_i)\otimes (x_3,y_3))$ (for $i=1,2$). 
As ${\cal O}( (x_1,y_1)\otimes (x_2,y_2)\otimes (x_3,y_3))=({\cal O}_l (x_1\otimes x_2\otimes x_3,y_1\otimes y_2\otimes y_3),{\cal O}_u (x_1\otimes x_2\otimes x_3,y_1\otimes y_2\otimes y_3))$, this means there are some $({\cal O}_l)_{i,3}$ and $({\cal O}_u)_{i,3}$ s.t.\ ${\cal O}_\dagger (x_1\otimes x_2\otimes x_3,y_1\otimes y_2\otimes y_3)_{|i,3}= ({\cal O}_\dagger)_{i,3}(x_i\otimes x_3,y_i\otimes y_3)$ for $\dagger \in \{l,u\}$ and $i=1,2$, which implies ${\cal L}^2_1\indep_{{\cal O}_l} {\cal L}^2_2\mid {\cal L}^2_3$ and ${\cal L}^2_1\indep_{{\cal O}_u} {\cal L}^2_2\mid {\cal L}^2_3$.

The $\Leftarrow$-direction is similar.
\end{appendixproof}

Stable operators respect the conditional independencies respected by the approximator from which they derive:
\begin{propositionAprep}\label{prop:complete:operators}
Let an approximation operator ${\cal O}$ over a bilattice of the complete product lattice $\bigotimes_{i\in \{1,2,3\}}{\cal L}_i$ be given s.t.\ ${\cal L}^2_1\indep_{\cal O} {\cal L}^2_2\mid {\cal L}^2_3$.
Then ${\cal L}_1\indep_{{\cal C}({\cal O}_l)} {\cal L}_2\mid {\cal L}_3$ and ${\cal L}_1\indep_{{\cal C}({\cal O}_u)} {\cal L}_2\mid {\cal L}_3$.
\end{propositionAprep}
\begin{appendixproof}
 As $L^2_1\indep_{\cal O} L^2_2\mid L^2_3$, with Proposition \ref{prop:conditional:independence:approx:components}, $L_1\indep_{{\cal O}_l} L_2\mid L_3$. As ${\cal O}_l(.,y)$ is $\leq$-monotonic (\cite[Proposition 6]{denecker2000approximations}), $\mathrm{lfp}({\cal O}_l(.,y))=\mathrm{lfp}(  ({\cal O}(.,y))_{1,3})\otimes  (\mathrm{lfp}(({\cal O}_l(.,y))_{2,3})_{|2})$ (Proposition \ref{prop:lfps:preserved}).
As $({\cal O}_l(x,y))=({\cal O}_l)_{1,3}(x_{|1,3},y_{|1,3})\otimes ({\cal O}_u)_{2,3}(x_{|2,3},y_{|2,3})_{|2}$ for any $x\in {\cal L}$ (in view of $L_1\indep_{{\cal O}_l} L_2\mid L_3$), we see that ${\cal O}_l(.,y)))_{1,3}=({\cal O}_l)_{1,3}(.,y_{|1,3})$. Thus, $\mathrm{lfp}({\cal O}_l(.,y))=\mathrm{lfp}(  ( {\cal O}_l)_{1,3})(.,y_{|1,3}))\otimes  \mathrm{lfp}(({\cal O}_l)_{2,3}(.,y_{|2,3}))_{|2}$. 
As 
${\cal C}({\cal O}_l)(y)=\mathrm{lfp}({\cal O}(.,y))$ for any $y\in {\cal L}$, this concludes the proof.
\end{appendixproof}
This allows us to derive another central result, stating that search for stable fixpoints, including the well-founded one, can be split up on the basis of conditional independencies.

\begin{propositionAprep}
Let an approximation operator ${\cal O}$ over a bilattice of the product lattice $\bigotimes_{i\in \{1,2,3\}} S_i$ be given s.t.\ ${\cal L}^2_1\indep_{\cal O} {\cal L}^2_2\mid {\cal L}^2_3$. Then:
\begin{enumerate}
\item $(x,y)$ is a fixpoint of $S({\cal O})$ iff 
$(x_{|i,3},y_{|i,3})$ is a fixpoint of $S({\cal O}_{i,3})$ for $i=1,2$.
\item $(x,y)$ is the well-founded fixpoint of ${\cal O}$ iff 
$(x_{|i,3},y_{|i,3})$ is the well-founded fixpoint of ${\cal O}_{i,3}$ for $i=1,2$.
\end{enumerate}
\end{propositionAprep}
\begin{appendixproof}
This follows immediately from Propositions \ref{prop:lfps:preserved} and \ref{prop:complete:operators}. 
\end{appendixproof}

\section{Parametrized Complexity Results Based on Conditional Independence}\label{sec:fpt}
In this section, we show how the notion of conditional independence can be made useful to break up reasoning tasks into smaller tasks that can be solved in parallel. We do this by introducing \textit{conditional independence trees}, and illustrate its usefulness by showing that the size of the induced modules serves as a parameter for the fixed parameter tractability of calculating the well-founded fixpoint.

\subsection{Preliminaries on Parametrized Complexity}
In this section, we recall the necessary background on (parametrized) complexity. For more detailed references, we refer to \cite{downey2013fundamentals}. We assume familiarity with basics on the polynomial hierarchy. An  instance of a \emph{parametrized problem L} is a pair $(I,k)\in \Sigma^\ast\times\mathbb{N}$ for some finite alphabet $\Sigma$, and we call $I$ the \emph{main part} and $k$ the \emph{parameter}.
Where $|I|$ denotes the cardinality of $I$, $L$ is \emph{fixed-parameter tractable} if there exists a computable function $f$ and a constant $c$ such that $(I,k)\in L$ is decidable in time $O(f(k)|I|^c)$. Such an algorithm is called a \emph{fixed-parameter tractable algorithm}. Thus, the intuition is that when $f(k)$ remains sufficiently small, the problem remains tractable no matter the size of $I$.

\subsection{Conditional Independence Trees}
How can conditional independence be used to make reasoning more efficient? Conditional independence allows to split up the application of an operator over lattice elements to parallel reasoning over elements of the sub-lattices.These sub-lattices thus define \emph{sub-modules} for parallel reasoning. 
Conditional independence only allows us to  split up reasoning into two modules, thus only splitting the search space in half at best. However, the operator allowing, these modules can again be split up into smaller modules, leading to a tree structure of nested modules, which we call \emph{conditional independence trees}:

\begin{definition}
Let $\bigotimes_{i\in I} {\cal L}_i$ be a product lattice. We call a binary labelled tree $(V,E,\nu)$ a \emph{conditional independence tree} for ${\cal O}$ (in short, CIT) if the following holds:
\begin{itemize}
\item $\nu:V\rightarrow 2^I\times 2^I\times 2^I$,
\item the root is labelled $\langle I_1,I_2,I_3\rangle$ where $I_1,I_2,I_3$ is a partition of $I$,
\item for every $(v_1,v_2),(v_1,v_3)\in E$, where $\nu(v_i)=\langle I^i_1,I^i_2,I^i_3\rangle$ for $i=1,2,3$, $I^1_j\cup I^1_3=I_1^{j+1}\cup I_2^{j+1}\cup I_3^{j+1}$ for $j=1,2$.
\item $\bigotimes_{i\in I_1}{\cal L}_i\indep_{O_{I_1\cup I_2\cup I_3}} \bigotimes_{i\in I_2}{\cal L}_i\mid \bigotimes_{i\in I_3}{\cal L}_i$ for every  $v\in V$ with $\nu(v)=\langle I_1, I_2, I_3\rangle$.
\end{itemize}
\end{definition}
Thus, a CIT is a tree where each vertex is labelled with a partition $\langle I_1,I_2,I_3\rangle$ of a sub-lattice $\bigotimes_{i\in I_1\cup I_2\cup I_3}\Lc_i$ s.t.\ $\bigotimes_{i\in I_1}\Lc_i$ is independent w.r.t.\ $\bigotimes_{i\in I_2}\Lc_i$ given $\bigotimes_{i\in I_3}\Lc_i$ according to $O$, and such that the labels of the leaves of a node compose the sub-lattices mentioned in the  the labels of the parent node.
Finally, the root of the tree should be a decomposition of the original lattice.

\begin{example}\label{example:cit:tree}
We consider an extension of Example \ref{example:infection:1} where ${\cal P}_2=\{r_1,r_2,r_3,r_4,r_5,r_6,r_7, r_8\}$ with:
\begin{eqnarray*}
&&r_6:\ \mathtt{inf}(\mathtt{d})\leftarrow \mathtt{inf}(\mathtt{c}),\mathtt{cnct}(\mathtt{c},\mathtt{d}), \mathrm{not}\: \mathtt{vac}(\mathtt{d}).\\
&&r_7:\ \mathtt{inf}(\mathtt{e})\leftarrow \mathtt{inf}(\mathtt{c}),\mathtt{cnct}(\mathtt{c},\mathtt{e}), \mathrm{not}\: \mathtt{vac}(\mathtt{e}).\\
&&r_8:\ \mathtt{cnct}(\mathtt{c},\mathtt{d}).,\: r_9:\: \mathtt{cnct}(\mathtt{c},\mathtt{e}).
\end{eqnarray*}

Furthermore, we let: 
\begin{eqnarray*}
{\cal A}_e & = &  \{\mathtt{inf}(\mathtt{e}),\mathtt{cnct}(\mathtt{c},\mathtt{e}),\mathtt{vac}(\mathtt{e})\}\\
{\cal A}_d & = & \{\mathtt{inf}(\mathtt{d}),\mathtt{cnct}(\mathtt{d},\mathtt{e}),\mathtt{vac}(\mathtt{d})\}\\
\end{eqnarray*}
The dependency-graph for this program is given in Figure \ref{fig:dependency:graph}. We observe the following independencies: 
\begin{eqnarray*}
&2^{{\cal A}_b}\bot_{\IC_{{\cal P}_2}} 2^{{\cal A}_c\cup {\cal A}_e\cup {\cal A}_d} \mid 2^{{\cal A}_a}\\
&2^{{\cal A}_e}\bot_{\IC_{{\cal P}_2}} 2^{{\cal A}_d} \mid 2^{{\cal A}_a\cup {\cal A}_c}
\end{eqnarray*}
We can accordingly obtain the following CIT:
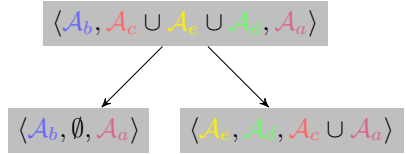
\begin{center}
\begin{tikzpicture}[->, >=stealth', node distance=2cm, scale=0.5]
    \node[fill=gray!50] (root) {$\langle {\color{blue!60!white}{\cal A}_b}, {\color{red!60!white}{\cal A}_c}\cup {\color{yellow}{\cal A}_e}\cup {\color{green!60!white}{\cal A}_d}, {\color{purple!60!white}{\cal A}_a}\rangle $};
    \node[fill=gray!50, below left of=root] (left) {$\langle {\color{blue!60!white}{\cal A}_b},\emptyset,{\color{purple!60!white}{\cal A}_a}\rangle$};
      \node[fill=gray!50, below right of=root] (right) {$\langle {\color{yellow}{\cal A}_e},{\color{green!60!white}{\cal A}_d},{\color{red!60!white}{\cal A}_c}\cup{\color{purple!60!white}{\cal A}_a}\rangle$};
    
    \path (root) edge (left);
    \path (root) edge (right);
\end{tikzpicture}
\end{center}
This CIT can be used to reduce global reasoning to parallel instances of local reasoning along the different components occuring in the leaves of the CIT. For example,
we can calculate
${\sf WF}({\cal P}_2)= {\sf WF}({\color{blue!60!white}{\cal P}_b}\cup {\color{purple!60!white}{\cal P}_a})\cup {\sf WF}({\color{yellow}{\cal P}_e}\cup {\color{red!60!white}{\cal P}_e}\cup  {\color{purple!60!white}{\cal P}_c})\cup {\sf WF}({\color{green!80!blue}{\cal P}_d}\cup {\color{red!60!white}{\cal P}_a}\cup  {\color{purple!60!white}{\cal P}_a})$ (where ${\cal P}_d=\{r_6, r_8\}$ and ${\cal P}_e=\{r_7, r_9\}$):
\begin{align*}
& {\sf WF}({\color{blue!60!white}{\cal P}_b}\cup {\color{purple!60!white}{\cal P}_a}) = \{\mathtt{inf}(\mathtt{a}), \mathtt{cnct}(\mathtt{a},\mathtt{b}), \mathtt{inf}(\mathtt{b})\}\\[1ex]
 &{{\sf WF}({\color{yellow}{\cal P}_e}\cup {\color{red!60!white}{\cal P}_c}\cup  {\color{purple!60!white}{\cal P}_a})}=\\
&\{\mathtt{inf}(\mathtt{a}), \mathtt{cnct}(\mathtt{a},\mathtt{c}), \mathtt{inf}(\mathtt{c}),  \mathtt{cnct}(\mathtt{c},\mathtt{e}), \mathtt{inf}(\mathtt{e})\}\\[1ex]
 & {\sf WF}({\color{green!80!blue}{\cal P}_d}\cup {\color{red!60!white}{\cal P}_c}\cup  {\color{purple!60!white}{\cal P}_a}) =\\
&\{\mathtt{inf}(\mathtt{a}), \mathtt{cnct}(\mathtt{a},\mathtt{c}), \mathtt{inf}(\mathtt{c}),  \mathtt{cnct}(\mathtt{c},\mathtt{d}), \mathtt{inf}(\mathtt{d})\}\\[1ex]
&{\sf WF}({\color{blue!60!white}{\cal P}_b}\cup
 {\color{red!60!white}{\cal P}_c}\cup
 {\color{purple!60!white}{\cal P}_a}\cup
 {\color{yellow}{\cal P}_e}\cup 
 {\color{green!80!blue}{\cal P}_d})=\\
 &{\sf WF}({\color{blue!60!white}{\cal P}_b}\cup {\color{purple!60!white}{\cal P}_a})\cup {\sf WF}({\color{yellow}{\cal P}_e}\cup {\color{red!60!white}{\cal P}_e}\cup  {\color{purple!60!white}{\cal P}_c})\cup {\sf WF}({\color{green!80!blue}{\cal P}_d}\cup {\color{red!60!white}{\cal P}_c}\cup  {\color{purple!60!white}{\cal P}_a})=\\
 &\{\mathtt{inf}(\mathtt{a}), \mathtt{cnct}(\mathtt{a},\mathtt{b}), \mathtt{inf}(\mathtt{b}), \mathtt{cnct}(\mathtt{a},\mathtt{c}), \\ &\mathtt{inf}(\mathtt{c}),  \mathtt{cnct}(\mathtt{c},\mathtt{e}), \mathtt{inf}(\mathtt{e}), \mathtt{cnct}(\mathtt{c},\mathtt{d}), \mathtt{inf}(\mathtt{d})\}
\end{align*}

\end{example}

The following result shows that a CIT-tree correctly decomposes reasoning problems in KR:
\begin{propositionAprep}\label{lemma:cit:gives:fps}
Let an operator $O$ over $\bigotimes_{i\in I} {\cal L}_i$ and CIT $T=(V,E,\nu)$ be given with $V_{\sf l}$ the leafs of $T$. Then:
\begin{enumerate}
\item $\lfp(O)=\bigotimes_{\langle I_1,I_2,I_3\rangle \in V_{\sf l}}\lfp(O_{I_1\cup I_3})\otimes \lfp(O_{I_2\cup I_3})$,
\item $x$ is a fixpoint of $O$ iff for every $\langle I_1,I_2,I_3\rangle \in V_{\sf l}$ and $i=1,2$,
 $x_{|I_i\cup I_3}$ is a fixpoint of $O_{I_i\cup I_3}$.
\end{enumerate}
\end{propositionAprep}
\begin{appendixproof}
We show this by induction on the size $|E|$ of CIT $T=(V,E,\nu)$ The base case is $|E|=0$ is trivial. For the inductive case, suppose the two statements hold for any CIT $T'=(V',E',\nu')$ with $|E'|\leq n$ and consider the CIT $T=(V,E,\nu)$ with $|E|= n+2$. We show the two claims for $T$:
\begin{enumerate}
\item Let $T'=(V',E',\nu')$ where
$V'=V\setminus V_{\sf l}$
$E'=E\cap (V'\times V')$, and
$\nu'(v)=\nu(v)$ for every $v\in V'$. Notice that $T'$ is also a CIT, and as $E'\subset E$, $|E'|\leq n$. Thus, with the inductive hypothesis, ($\dagger$): $\lfp(O)=\bigcup_{\langle I_1,I_2,I_3\rangle \in V'_{\sf l}}\lfp(O_{I_1\cup I_3})\otimes \lfp(O_{I_2\cup I_3})$, where $V'_{\sf l}$ are the leaf nodes of $V'$. Consider now some $v\in V'_{\sf l}$ with children nodes $v_1$ and $v_2$ and $\nu(v)=\langle I_1,I_2,I_3\rangle$. 
By definition of a CIT, where $\nu(v_1)=\langle I_1^1,I_2^1,I_3^1\rangle$ and  $\nu(v_1)=\langle I_1^2,I_2^2,I_3^2\rangle$, $\bigotimes_{i\in I^j_1} {\cal L}_i\bot_{{O}_{I^j_1\cup I^j_2\cup I^j_3}}\bigotimes_{i\in I^j_2} {\cal L}_i\mid \bigotimes_{i\in I^j_3} {\cal L}_i$ for $j=1,2$, which means (with Proposition \ref{prop:lfps:preserved}), that $\lfp(O_{I^j_1\cup I^j_2\cup I^j_3})=\lfp(O_{I^j_1\cup  I^j_3})\otimes \lfp(O_{I^j_2\cup  I^j_3})$ for $j=1,2$. 
As $I_1^1\cup I_2^1\cup I_3^1=I_1\cup I_3$ and $I_1^2\cup I_2^2\cup I_3^2=I_2\cup I_3$ and $\bigotimes_{i\in I_1} {\cal L}_i\bot_{{O}_{I_1\cup I_2\cup I_3}}\bigotimes_{i\in I_2} {\cal L}_i\mid \bigotimes_{i\in I_3} {\cal L}_i$, we conclude that ($\ddagger$): $\lfp(O_{I_1\cup I_2\cup I_3})=   \lfp(O_{I_1\cup I_3})\otimes \lfp(O_{I_2\cup I_3})=\lfp(O_{I^1_1\cup I^1_2\cup I^1_3})\otimes \lfp(O_{I^2_1\cup I^2_2\cup I^2_3})$. As this line of reasoning holds for any $v'\in V'_{\sf l}$, and  ($\star$): $V_{\sf l}=\{ v\in V\mid v \mbox{ is a child of some }v'\in V'_{\sf l}\}$, we see that:
\begin{align*}
\lfp(O)&= \bigcup_{\langle I_1,I_2,I_3\rangle \in V'_{\sf l}}\lfp(O_{I_1\cup I_3})\otimes \lfp(O_{I_2\cup I_3})&\quad (\dagger)\\
&=\bigcup_{v \in V'_{\sf l}} \{\lfp(O_{I^1_1\cup I^1_2\cup I^1_3})\otimes \lfp(O_{I^2_1\cup I^2_2\cup I^2_3})\mid \langle I^i_1,I^i_2,I^i_3\rangle \mbox{ are the children of } v \} &\quad (\ddagger)\\
&= \bigcup_{\langle I_1,I_2,I_3\rangle \in V_{\sf l}}\lfp(O_{I_1\cup I_3})\otimes \lfp(O_{I_2\cup I_3})&\quad (\star)
\end{align*}

\item Similar to the proof of the previous item, by replacing Proposition \ref{prop:lfps:preserved} by Proposition \ref{prop:fixed:points:preserved}. 
\end{enumerate}
\end{appendixproof}

We show how the decomposition of reasoning according to a CIT results in FPT-results, making the following 
\begin{assumption}
In the rest of this section, we assume that $L$ is a powerset-lattice, i.e.\ $L=\langle 2^S,\subseteq\rangle$ for some  set $S$, to ensure that cardinality used in Def.\ \ref{def:cit:partition:size} is defined. 
\end{assumption}
Lifting this assumption to other lattices is straightforward, but requires one to define a notion of size on the lattice elements. We leave this for future work. We furthermore notice that we will use $|S|$ as the input size of the problems studied below. Notice that we do \emph{not} consider $|2^S|$ as the input size, as this would make most results below trivial.

 \emph{CIT-size} encodes the size of modules and serving as basis for the parameter in the FPT-results below.
\begin{definition}\label{def:cit:partition:size}
Let an operator $O$ over $\bigotimes_{i\in I} {\cal L}_i$, a CIT $T=(V,E,\nu)$ for $O$ and $V_{\sf l}$ the leafs of $T$ be given. The \emph{CIT-partition-size} of $O$ relative to $(V,E,\nu)$ is defined 
as 
$${\sf CPS}(T)=\max\left(\{|\bigotimes_{i\in I_j}{\cal L}_i\otimes\bigotimes_{i\in I_3}{\cal L}_i|\mid v\in V_{\sf l}, \nu(v)=\langle I_1,I_2,I_3\rangle, j=1,2\}\right).$$
 The \emph{CIT-size} of $O$ relative to $T$ is defined 
as ${\sf CS}(V,E,\nu)=\max( {\sf CPS}(V,E,\nu), |V_{\sf l}|)$.
\end{definition}
\begin{example}
In the CIT of Example \ref{example:cit:tree}, the CIT-size is equal to the size of the largest partition, namely $2^{|{\cal A}_e\cup{\cal A}_c\cup {\cal A}_a|}= 2^{|{\cal A}_d\cup{\cal A}_c\cup {\cal A}_a|}= 2^7$.
\end{example}
\nocite{cousot1979constructive}

\begin{toappendix}
\begin{remark}
The CIT-partition-size alone is not sufficient as a parameter, thus necessitating the consideration of $ |V_{\sf l}|$ in the CIT-size. The reason is that sub-lattices can occur arbitrarily many times in a CIT, which might lead to redundancies:
\begin{example}
Consider a logic program ${\cal P}=\{p_i.\mid i=1\ldots n\}$ for some arbitrary but fixed $n\in\mathbb{N}$.
Notice that for any three disjoint sets ${\cal A}_1,{\cal A}_2\subseteq \{p_i\mid i=1\ldots n\}$, it holds that
 $2^{{\cal A}_1}\bot_{\IC_{{\cal P}}} \emptyset\mid 2^{{\cal A}_2}$.
We show, for $n=3$, how we can produce a CIT with more than $2^n$ different nodes. This recipe can be generalized for any $n$. 

 \begin{center}
\begin{tikzpicture}[->, >=stealth', node distance=2.7cm]
    \node[fill=gray!50, scale=0.7] (root) {$\langle \emptyset,\emptyset,\{p_1,p_2,p_3\}\rangle$};
    \node[scale=0.7, fill=gray!50, below left of=root] (left1) {$\langle \emptyset,\emptyset,\{p_1,p_2,p_3\}\rangle$};
      \node[scale=0.7, fill=gray!50, below right of=root] (right1) {$\langle \emptyset,\emptyset,\{p_1,p_2,p_3\}\rangle$};
         \node[scale=0.7, fill=gray!50, below left of=left1] (leftleft1) {$\langle \emptyset,\emptyset,\{p_1,p_2,p_3\}\rangle$};
      \node[scale=0.7, fill=gray!50, below of=left1] (leftright1) {$\langle \emptyset,\emptyset,\{p_1,p_2,p_3\}\rangle$};
         \node[scale=0.7, fill=gray!50, below of=right1] (rightleft1) {$\langle \emptyset,\emptyset,\{p_1,p_2,p_3\}\rangle$};
      \node[scale=0.7, fill=gray!50, below right of=right1] (rightright1) {$\langle \emptyset,\emptyset,\{p_1,p_2,p_3\}\rangle$};
      
\node[scale=0.7, fill=gray!50, below left of=leftleft1] (leftleftleft1) {$\langle \{p_1\},\{p_2\},\emptyset \rangle$};
\node[scale=0.7, fill=gray!50, below of=leftleft1] (rightleftleft1) {$\langle \{p_2\},\{p_3\},\emptyset\rangle$};

\node[scale=0.7, fill=gray!50, below of=leftright1] (leftleftright1) {$\langle \{p_2\},\{p_1\},\emptyset \rangle$};
\node[scale=0.7, fill=gray!50, below right of=leftright1] (rightleftright1) {$\langle \{p_3\},\{p_2\},\emptyset\rangle$};

\node[scale=0.7, fill=gray!50, below of=rightleft1] (leftrightleft1) {$\langle \{p_1\},\{p_3\},\emptyset \rangle$};
\node[scale=0.7, fill=gray!10, below right of=rightleft1] (rightrightleft1) {$\langle \{p_2\},\{p_3\},\emptyset\rangle$};

\node[scale=0.7, fill=gray!50, below right of=rightright1] (leftrightright1) {$\langle \{p_3\},\{p_1\},\emptyset \rangle$};
\node[scale=0.7, fill=gray!10, right of=leftrightright1] (rightrightright1) {$\langle \{p_3\},\{p_2\},\emptyset\rangle$};

    \path (root) edge (left1);
    \path (root) edge (right1);
    \path (left1) edge (leftleft1);
    \path (left1) edge (leftright1);
    \path (right1) edge (rightleft1);
    \path (right1) edge (rightright1);
    \path (leftleft1) edge (leftleftleft1);
    \path (leftleft1) edge (rightleftleft1);
    \path (leftright1) edge (leftleftright1);
    \path (leftright1) edge (rightleftright1);
    \path (rightleft1) edge (leftrightleft1);
    \path (rightleft1) edge (rightrightleft1);
    \path (rightright1) edge (leftrightright1);
    \path (rightright1) edge (rightrightright1);

\end{tikzpicture}
\end{center}
Notice that we have, discounting the double (light grey) nodes, $2^3=8$ non-identical nodes. Thus, even though we have only three atoms, solving any problem along this CIT would result in $8$ different nodes being touched.

\end{example}
\end{remark}
\end{toappendix}

We first show that computing the well-founded fixpoint is FPT with the CIT-partition-size as a parameter. We do this under the assumption that applying an operator $O$ can be reduced by a call to an ${\sf NP}$-oracle. This covers several interesting cases, e.g.\ ADFs \citep{strass-wallner15complexity} and the ultimate operator \citep{denecker2004ultimate} for normal logic programs (see also Section \ref{sec:conditional:independence:lp}).
\begin{propositionAprep}\label{prop:fpt:of:wf}
Let a $\leq_\otimes$-monotonic operator $O$ over the powerset lattice $2^S$ and a CIT $T$ for $O$ with CIT-size $c$ be given. Assume that $O(x)$ can be computed by a call to an ${\sf NP}$-oracle for any $x\subseteq S$. The least fixpoint of $O$ can be computed in time $O(f(c).|S|)$.
\end{propositionAprep}
\begin{appendixproof}
Consider a powerset lattice $2^S$ where $S=\{s_i\mid i\in I\}$. 
Let  $T=(V,E,\nu)$ be a CIT for $O$, let  $V_{\sf l}$ are the leafs of $T$ and let $c$ be the CIT-partition-size.
We know that for every  $v\in V_{\sf l}$, where $\nu(v)=\langle I_1,I_2,I_3\rangle$ (slightly abusing notation for the indices), 
$2^{\{s_i\mid i\in I_1\}}\indep_{{\cal O}_{I_1\cup I_2\cup I_3}} 2^{\{s_i\mid i\in I_2\}}\mid 2^{\{s_i\mid i\in I_3\}}$.
As $O$ is $\leq_\otimes$-monotonic,  $O_{I_i\cup I_3}$ is also $\leq_\otimes^{I_i\cup I_3}$-monotonic (Proposition \ref{prop:monotonicity:preserved}) and thus the least fixpoint of $O_{I_i\cup I_3}$ can be computed by a linear number ($|I_i\cup I_3|+1$ to be precise) of calls to an ${\sf NP}$-oracle (in view of Theorem 5.1 by \citet{cousot1979constructive}), where every call is bounded by the size of $|2^{\{s_i\mid i\in I_i\cup I_3\}}|$ (which is equal or smaller than the CIT-size $c$). 
In other words, we  can solve this problem by less than $|I_1|+1$ SAT-calls of size smaller than $c$.
We repeat this proces for every leaf node, i.e.\ $|V_{\sf l}|\leq c$ times. This results in no more than  $(\max(\{|I_i\cup I_3| \mid \langle I_1,I_2,I_3\rangle \in V_{\sf l}, i=1,2\})+1).c$ SAT-calls of size smaller than $c$.
We now have obtained the lfp of $O_{I_1\cup I_3}$ respectively $O_{I_2\cup I_3}$ for every $\langle I_1,I_2,I_3\rangle\in V_{\sf l}$. All that remains is to take the union of these, as (in view of Proposition \ref{lemma:cit:gives:fps}): $$\lfp(O)=\bigcup_{\langle I_1,I_2,I_3\rangle\in V_{\sf l}} \lfp(O_{I_1\cup I_3})\cup \lfp(O_{I_2\cup I_3})$$

\end{appendixproof}

We now show similar results for skeptical and credulous inference, under the assumption that $O(x)$ is polynomial.

\begin{propositionAprep}\label{prop:fpt:of:credulous}
Let a $\leq_\otimes$-monotonic operator $O$ over the powerset lattice $2^S$, a CIT $T$ for $O$ with CIT-size $c$ and some $\alpha\in S$  be given. Assume that $O(x)$ can be computed in polynomial time for any $x\subseteq S$:
\begin{enumerate}
\item Determining whether there exists some fixpoint $x$ of $O$ s.t.\ $\alpha\in x$ can be done in  time $O(f(c).|S|)$.

\item Determining whether for every fixpoint $x$ of $O$ it holds that $\alpha\in x$ can be done in  time $O(f(c).|S|)$.
\end{enumerate} 
\end{propositionAprep}
\begin{appendixproof}
Ad 1.\ Consider a powerset lattice $2^S$ where $S=\{s_i\mid i\in I\}$. 
Let  $T=(V,E,\nu)$ be a CIT for $O$, let  $V_{\sf l}$ are the leafs of $T$ and let $s$ be the CIT-size.
We know that for every  $v\in V_{\sf l}$, where $\nu(v)=\langle I_1,I_2,I_3\rangle$,
$2^{\{s_i\mid i\in I_1\}}\indep_{{\cal O}_{I_1\cup I_2\cup I_3}} 2^{\{s_i\mid i\in I_2\}}\mid 2^{\{s_i\mid i\in I_3\}}$.
Let $V_\alpha\subseteq V_{\sf l}$ be the set of leaf-nodes s.t.\ $\nu(v)=\langle I_1,I_2,I_3\rangle$ and $\alpha\in \{s_i\mid i\in I_1\cup I_2\cup I_3\}$ (i.e.\ $V_\alpha$ contains all the modules that contain $\alpha$). With Proposition \ref{prop:fixed:points:preserved}, it suffices to check whether $O_{I_i\cup I_3}$ admits a fixpoint containing $\alpha$ (for $i=1,2$ s.t.\ $\alpha\in \bigotimes_{i\in I_1\cup I_2\cup I_3}\Lc_i$ for every $v\in V_\alpha$).
 We can do this exhaustively by verifying for every $x_i\otimes x_3\subseteq\{s_i\mid i\in I_i\cup I_3\}$ whether it is a fixpoint. This requires maximally $2.|V_{\sf l}|.2^{|2^{\{s_i\mid i\in I_i\cup I_3\}}|}$ checks of polynomial time, i.e.\ can be done in polynomial time bounded $c$.

Ad 2.\ Similar to the proof of the previous item.
\end{appendixproof}

It should be noticed that we do not say anything here about the complexity of determining the CIT-partition-size. This is because, firstly,  the generality of the operator-based framework makes it hard to make concrete claims about this very formalism-dependent question. Secondly, for concrete operators, efficiently looking for conditional independencies is a research topic on itself, left for future work.

\section{Application to Logic Programs}
\label{sec:conditional:independence:lp}
In this section, we illustrate  the theory developed in the previous section by applying it to normal logic programs. We can avoid clutter with a slight abuse of notation by writing ${{\cal A}_1}\indep_{\cal P} {{\cal A}_2}\mid {\cal A}_3$ to denote $2^{{\cal A}_1}\indep_{\ICc_{\cal P}} 2^{{\cal A}_2}\mid 2^{{\cal A}_3}$.

We first define what we call the \emph{marginalisation} of a program w.r.t.\ a set of atoms, which represents the syntactic counterpart of the modularised operator $O_{i,3}$. 
\begin{definition}
Let an nlp ${\cal P}$ and some ${\cal A}\subseteq {\cal A}_{\cal P}$ be given. We define ${\cal P}_{\cal A}$ as the program obtained by replacing in every rule $r\in {\cal P}$ every occurrence of an atom $p\in {\cal A}$ by $\bot$.
\end{definition}
For example, $\{p\leftarrow q,r,\lnot s\}_{\{r,s\}}=\{p\leftarrow q,\bot,\top\}$.
This transformation is correct, i.e.\ given a program ${\cal P}$ with  ${{\cal A}_1}\indep_{\cal P} {{\cal A}_2}\mid {\cal A}_3$, the marginalisation ${\cal P}_{{\cal A}_2}$ induces the ${\cal IC}$-operator for the sublattice ${\cal A}_1\cup{\cal A}_3$:
\begin{propositionAprep}
Let a nlp ${\cal P}$ be given for which ${\cal A}_{\cal P}$ is partitioned into ${\cal A}_1\cup {\cal A}_2\cup {\cal A}_3$ s.t.\ ${{\cal A}_1}\indep_{\cal P}{\cal A}_2\mid {\cal A}_3$. Then ${\cal IC}^{i,3}_{{\cal P}}={\cal IC}_{{\cal P}_{{\cal A}_j}}$ (for $i,j=1,2$, $i\neq j$).
\end{propositionAprep}
\begin{appendixproof}
 Consider some arbitrary but fixed $i,j=1,2$, $i\neq j$.
 We first notice that, in view of ${{\cal A}_1}\indep_{\cal P} {{\cal A}_2}\mid {\cal A}_3$,  ${\cal IC}^{i,3}_{{\cal P}}(x_i\cup x_3)={\cal IC}_{{\cal P}}(x_i\cup x_j\cup x_3)_{|i,3}$ for any $x_j\subseteq {\cal A}_j$. In particular, ${\cal IC}^{i,3}_{{\cal P}}(x_i\cup \emptyset \cup x_3)={\cal IC}_{{\cal P}}(x_i\cup \emptyset \cup x_3)_{|i,3}$ for any $x_i\subseteq {\cal A}_i$.
 
We first show that  ${\cal IC}^{i,3}_{{\cal P}}(x_i\cup x_3)={\cal IC}_{{\cal P}}(x_i\cup \emptyset \cup x_3)_{|i,3}\subseteq{\cal IC}_{{\cal P}_{{\cal A}_j}}(x_i\cup x_3)$ for any $x_i\cup x_3\subseteq {\cal A}_i\cup{\cal A}_3$. 
Indeed, suppose that $p\in {\cal IC}_{{\cal P}}(x_i\cup \emptyset \cup x_3)_{|i,3}$, i.e.\ $p\in {\cal A}_i\cup {\cal A}_3$ and there is some $r: p\leftarrow p_1,\ldots, p_n, \lnot q_1,\ldots,\lnot q_m\in {\cal P}$ s.t.\  $p_l\in x_i\cup x_3$ for every $l=1,\ldots,n$ and $q_l\not\in x_i\cup x_3$ for every $l=1,\ldots,n$. 
This implies that $\{p_1,\ldots,p_n\}\subseteq x_i\cup x_3$, which means that the transformation of $r$ in ${\cal P}_{{\cal A}_j}$ is  $r': p\leftarrow p_1,\ldots, p_n, \lnot q'_1,\ldots,\lnot q'_m$ where $q'_l=\top$ if $q_l\in {\cal A}_j$ and $q'_l=q_l$ otherwise. But then clearly $p$ is derivable from $x_i\cup x_3$ using $r'$. 
 
 We now show that  ${\cal IC}^{i,3}_{{\cal P}}(x_i\cup x_3)={\cal IC}_{{\cal P}}(x_i\cup \emptyset \cup x_3)_{|i,3}\supseteq{\cal IC}_{{\cal P}_{{\cal A}_j}}(x_i\cup x_3)$. Suppose that $p\in {\cal IC}_{{\cal P}_{{\cal A}_j}}(x_i\cup x_3)$. Then there is some $r': p\leftarrow p_1,\ldots, p_n, \lnot q_1,\ldots,\lnot q_m\in {\cal P}_{{\cal A}_j}$ s.t.\ $p_1,\ldots,p_n\in x_i\cup x_3$ and $q_1,\ldots,q_n\not\in x_1\cup x_3$. Also, there is some $r: p\leftarrow p'_1,\ldots, p'_n, \lnot q'_1,\ldots,\lnot q'_m\in {\cal P}$ s.t.\ $p'_l=p_l$ if $p_l\in {\cal A}_i\cup {\cal A}_3$ and $p'_l=\bot$ otherwise (for every $l=1,\ldots,m$), and similarly for $q_l$. As $p$ was derived using $r'$, clearly, $p_l\in {\cal A}_i\cup{\cal A}_3$ for every $l=1,\ldots,n$. Thus, $p$ is derivable from $x_i\cup x_3$ using $r'$, which establishes $p\in {\cal IC}_{{\cal P}}(x_i\cup \emptyset \cup x_3)_{|i,3}$.

\end{appendixproof}

We first observe that the search for supported, (partial) stable and well-founded models can be split up along conditionally independent sub-alphabets:
\begin{corollary}
Let an nlp ${\cal P}$ be given for which ${\cal A}_{\cal P}$ is partitioned into ${\cal A}_1\cup {\cal A}_2\cup {\cal A}_3$ s.t.\ ${{\cal A}_1}\indep_{\cal P}{\cal A}_2\mid {\cal A}_3$. 
$x_1\cup x_2\cup x_3$ is a supported (respectively three-valued stable) model of ${\cal P}$ iff $x_i\cup x_3$ is a supported  (respectively three-valued stable) model of ${\cal P}_{A_j}$ (for $i,j=1,2$ and $i\neq j$). 
The well-founded model of ${\cal P}$ can be obtained as $(x_1\cup x_2\cup x_3,y_1\cup y_2\cup y_3)$, where $(x_i\cup x_3,y_i\cup y_3)$ is the well-founded model of ${\cal P}_{{\cal A}_j}$ (for $i,j=1,2$, $i\neq j$).
\end{corollary}

 It is well-known that computing the well-founded fixpoint for the operator ${\cal IC}_{\cal P}$ can be done in polynomial time.
In the appendix, we also show that computing the well-founded fixpoint for the ultimate operator \citet{denecker2002ultimate}, which is ${\sf NP}$-hard, is FPT for the CIT-size (as a corollary of Proposition \ref{prop:fpt:of:wf}). Likewise, Proposition \ref{prop:fpt:of:credulous} can be used to show a similar result for credulous and skeptical entailment under the stable model semantics.

\begin{toappendix}

\begin{corollary}
Given a normal logic program ${\cal P}$ and a  CIT $T=(V,E,\nu)$ with CIT-partition-size $s$ for ${\cal P}$. Then the well-founded fixpoint of ${\cal IC}_{\cal P}^{{\sf DMT}}$ can be computed in time $O(f(s).|{\cal A}_{\cal P}|)$.
\end{corollary}
\begin{proof}
Immediate in view of Proposition \ref{prop:fpt:of:wf}.
\end{proof}
\end{toappendix}

\newcommand{\dep}{\mathrm{DP}}

We now make some observations on how to detect conditional independencies in a logic program based on its syntax. We first need some further preliminaries.
The \emph{dependency order} for a logic program ${\cal P}$, $\leq^{\cal P}_{\rm dep}\subseteq {\cal A}_{\cal P}\times {\cal A}_{\cal P}$, is defined as $p\leq^{\cal P}_{\rm dep} q$ iff there is some $r\in {\cal P}$ where $q$ is the head of $r$ and $p$ occurs in the body of $r$. The \emph{dependency graph}, denoted $\dep({\cal P})$ of ${\cal P}$ is the Hasse diagram of $\leq^{\cal P}_{\rm dep}$. Figure \ref{fig:dependency:graph} is an example of the (inverse of) a dependency graph.

A first conjecture could be that, a sufficient criterion fo ${\cal A}_1\indep_{\cal P} {\cal A}_2\mid {\cal A}_3$ is that ${\cal A}_3$ graphically separates ${\cal A}_1$ and ${\cal A}_2$, i.e.\ 
 given $\dep({\cal P})=\langle {\cal A}_{\cal P}, V\rangle$, a set ${\cal A}_3$ s.t.\ $\langle {\cal A}_{\cal P}\setminus {\cal A}_3,V\cap ( ({\cal A}_{\cal P}\setminus ({\cal A}_3)\times ({\cal A}_{\cal P}\setminus ({\cal A}_3))$ consists of two disconnected subgraphs $\langle {\cal A}_1,V\cap ({\cal A}_1\times {\cal A}_1)\rangle$ and $\langle {\cal A}_2,V\cap ({\cal A}_2\times {\cal A}_2)\rangle$ induces the conditional independence ${\cal A}_1\indep_{\cal P} {\cal A}_2|{\cal A}_3$. However, this conjecture is too naive:
\begin{example}\label{example:graphical:criterion:too:naive}
Consider the program ${\cal P}=\{ a_1\leftarrow \lnot\: b_1; b_1\leftarrow \lnot\: a_1; e\leftarrow b_1\}$ and ${\cal P}_2=\{ a_2\leftarrow \lnot\: b_2; b_2\leftarrow \lnot\: a_2; e\leftarrow b_2\}$. This program has the following dependency graph: 

\begin{tikzpicture}[->, >=stealth', node distance=1.5cm, scale=0.5]
    \node (e) {e};
    \node[left of=e] (b1) {$b_1$};
    \node[left of=b1] (a1) {$a_1$};
    \node[right of=e] (b2) {$b_2$};
    \node[right of=b2] (a2) {$a_2$};

    \path (b1) edge (e);
        \path (b1) edge (a1);
        \path (a1) edge (b1);
 \path (b2) edge (e);
        \path (b2) edge (a2);
        \path (a2) edge (b2);
\end{tikzpicture}

We could conjecture $\{a_1,b_1\}\indep_{\cal P} \{a_2,b_2\}|\{e\}$, but this does not hold, as
\begin{eqnarray*}
&\IC_{\cal P}(\{a_1,b_2\})_{|\{a_1,b_1,e\}}&=\{a_1,e\}
 \\ &\neq 
\IC_{\cal P}(\{a_1\})_{|\{a_1,b_1,e\}}&=\{a_1\}.
\end{eqnarray*}
\end{example}

A slightly more complicated graphical criterion is a sufficient condition, though. In more detail, if ${\cal A}_3$ graphically seperates ${\cal A}_1$ and ${\cal A}_2$ in $\dep({\cal P})$, and if the program is stratified in a lower layer ${\cal A}_3$ and a higher layer ${\cal A}_1\cup {\cal A}_2$, then the conditional independency ${\cal A}_1\indep_{\cal P}{\cal A}_2\mid {\cal A}_3$ holds: 
\begin{propositionAprep}\label{prop:syntactic:cond:indep:lp}
Let a nlp ${\cal P}$ with $\dep({\cal P})=\langle {\cal A}_{\cal P}, V\rangle$  be given s.t.\ the following conditions hold
\begin{enumerate}
\item there is some ${\cal A}_3\subseteq {\cal A}_{\cal P}$ s.t.\  $\langle {\cal A}_{\cal P}\setminus {\cal A}_3,V\cap ( ({\cal A}_{\cal P}\setminus ({\cal A}_3)\times ({\cal A}_{\cal P}\setminus ({\cal A}_3))$ consists of two disconnected subgraphs $\langle {\cal A}_1,V\cap ({\cal A}_1\times {\cal A}_1)\rangle$ and $\langle {\cal A}_2,V\cap ({\cal A}_2\times {\cal A}_2)\rangle$, and
\item for every $a\in {\cal A}_3$ and $b\in {\cal A}_i$ ($i=1,2$), $b<_{\rm dep} a$.
\end{enumerate}
Then ${\cal A}_1\indep_{\cal P}{\cal A}_2\mid {\cal A}_3$ holds.
\end{propositionAprep}
\begin{appendixproof}
Suppose the conditions of the proposition hold. Then for any $p\leftarrow \bigwedge \Delta\land \bigwedge \Theta^\lnot$ (where $\Theta^\lnot=\{\lnot p\mid p\in \Theta\}$), (1) $\Delta\cup\Theta\cap {\cal A}_i\neq\emptyset$ implies $p\in {\cal A}_i$ (for $i=1,2$), and (2) $\Delta\cup\Theta\cup\{p\}\subseteq {\cal A}_i\cup{\cal A}_3$ (for $i=1,2$). From (1), it follows that $\dagger$: $\ICc_{\cal P}(x_1\cup x_2\cup x_3)_{|3}=\IC_{{\cal P}_{{\cal A}_1\cup {\cal A}_2}}(x_3)$ for any $x_i\subseteq {\cal A}_i$ ($i=1,2,3$). From (2) and $\dagger$, it then follows that: $\ICc_{\cal P}(x_1\cup x_2\cup x_3)_{|i,3}=\ICc_{{\cal P}_{{\cal A}_k}}(x_i\cup x_3)$ for any $x_j\subseteq {\cal A}_j$ ($j=1,2,3)$ and $i,k=1,2$ and $i\neq k$.
\end{appendixproof}
A case in point of the criteria in this proposition is Example \ref{example:infection:1}. The search for more comprehensive, potentially even necessary, criteria for identifying conditional independencies are an avenue for future work.

\section{Related Work}
\label{sec:rel:work}
In this section, we provide a summary of related work. For most of the comparisons, we provide more formal comparisons in the appendix.
In the context of classical logic, a notion of conditional independence was proposed by \citet{darwiche1997logical}. Darwiche assumes a database $\Delta$ (i.e.\ a set of propositional formulas), which is used as a background theory for inferences. The idea behind conditional independence is then that a database $\Delta$ sanctions the independence of two sets of atoms $x_1$ and $x_2$ conditional on a third set of atoms $x_3$ if, given full information about $x_3$, inferences about $x_1$ are independent from any information about $x_2$. 
As shown in the appendix, our notion of conditional independence implies Darwiche's notion.

A concept related to conditional independence studied in approximation fixpoint theory is that of \emph{stratification} \citep{vennekens2006splitting}. This work essentially generalizes the idea of \emph{splitting} as known from logic programming, where the idea is to divide a logic program in layers such that computations in a given layer only depend on rules in the layer itself or layers below. For example, the program $\{q\leftarrow \sim r; r\leftarrow \sim s; s\leftarrow \sim p\}$ can be stratified in the layers $\{p\},\{s,r\},\{q\}$. This concept was formulated purely algebraically by \citet{vennekens2006splitting}. Our study of conditional independence took inspiration from this work in using product lattices as an algebraic tool for dividing lattices, and many proofs and results in our paper are similar to those shown for stratified operators \citep{vennekens2006splitting}. Conceptually, stratification and conditional independence seem somewhat orthogonal, as conditional independence allows to divide a lattice ``horizontally'' into independent parts, whereas stratification allows to divide a lattice ``vertically'' in layers that incrementally depend on each other. However, we show in the appendix that conditional independence can be seen as a special case of stratification. 

A lot of work exists on the \emph{parametrization} of the computational complexity of various computational tasks using \emph{treewidth decompositions} as a parameter \citep{gottlob2002fixed}. These results show that the computational effort required in solving a problem is not a function of the overall size of the problem, but rather of certain structural parameters of the problem, i.e.\ the treewidth of a certain representation of the problem. 
These techniques have been applied to answer set programming \citep{fichte2017answer}. In these works, the treewidth of the tree decomposition of the dependency graph ${\sf DP}({\cal P})$ and incidence graph (which also contains vertices for rules) of a logic program are used as parameters to obtain fixed-parameter tractability results. In the appendix, we show conditional independence and treewidth are orthogonal. Furthermore, our results  go beyond the answer set semantics, e.g.\ partial stable, supported and well-founded (ultimate) semantics.

\nocite{makinson2003input}
\begin{toappendix}
 We first discuss Darwiche's notion of conditional independence \citep{darwiche1997logical}, stratification as studied in AFT \citep{vennekens2006splitting} and treewidth-based decompositions of logic programs in detail, and then make shorter comparisons to other related works.

\paragraph*{Darwiche's Logical Notion of Independence}
In the context of classical logic, a notion of conditional independence was proposed by \citet{darwiche1997logical}. Darwiche assumes a database $\Delta$ (i.e.\ a set of propositional formulas), which is used as a background theory for inferences. The idea behind conditional independence is then that a database $\Delta$ sanctions the independence of two sets of atoms $x_1$ and $x_2$ conditional on a third set of atoms $x_3$ if, given full information about $x_3$, inferences about $x_1$ are independent from any information about $x_2$. 
In other words, 
given a set of formulas $\Delta$ and three disjoint sets of atoms $x_1,x_2$ and $x_3$ be given, 
$x_1\indep_{\Delta}^{\sf D} x_2\mid x_3$ iff for every formula $\phi_1$ based on $x_1$, $\phi_2$ based on $x_2$ and complete conjunction of literals $\phi_3$ based on $x_3$ s.t.\ $\Delta\cup\{\phi_3,\phi_2\}$ is consistent, the following holds:
\[ \Delta\cup\{\phi_3\}\models \phi_1 \quad \mbox{iff}\quad \Delta\cup\{\phi_2,\phi_3\}\models \phi_1\]
Even though the application of our notion of conditional independence to operators ranging over sets of possible worlds (which is required to give an operator-based characterisation of propositional logic) is outside the scope of this paper, we can nevertheless show a close connection between our notion of conditional independence and the one formulated by Darwiche by defining inference based on a logic program as follows (which gives rise to a special case of \emph{simple-minded output} as known from input/output logics \citep{makinson2003input}):
\begin{definition}
Given a logic program ${\cal P}$ and formulas $\phi,\psi$ based on ${\cal A}_{\cal P}$, we define: $\phi\models_{\cal P} \psi$ if for every $x\subseteq {\cal A}_{\cal P}$ s.t.\ $x(\phi)={\sf T}$, $\IC_{\cal P}(x)(\psi)={\sf T}$.
\end{definition}

We can now show that our notion of conditional independence implies Darwiche's notion of conditional independence, interpreted in the setting of inference based on logic programs:
\begin{propositionAprep}
Let a program ${\cal P}$ for which ${\cal A}_{\cal P}$ is partitioned into ${\cal A}_1\cup {\cal A}_2\cup {\cal A}_3$ s.t.\ ${{\cal A}_1}\indep_{\cal P}{\cal A}_2\mid {\cal A}_3$, some $\phi_1$ based on ${\cal A}_1$, some $\phi_2$ based on ${\cal A}_2$ and a complete conjunction of literals $\phi_3$ based on ${\cal A}_3$ be given. Then
$\phi_3\models_{\cal P} \phi_1$ iff $\phi_3\land \phi_2\models_{\cal P}\phi_1$.
\end{propositionAprep}
\begin{proof}
Suppose that the assumptions of this proposition holds. The $\Rightarrow$-direction is immediate as $\models_{\cal P}$ is monotonic.

Suppose now that  $\phi_3\land \phi_2\models_{\cal P}\phi_1$. Then for every $x_1\cup x_2\cup x_3 \subseteq {\cal A}_{\cal P}$ s.t.\ $x_1\cup x_2\cup x_3(\phi_3\land \phi_2)={\sf T}$, $\IC_{\cal P}(x_1\cup x_2\cup x_3)(\phi_1)={\sf T}$. Notice that there is a single $x_3\subseteq {\cal A}_3$ s.t.\ $x_3(\phi_3)$ and $x_1\cup x_2\cup x_3(\phi_3\land \phi_2)={\sf T}$ is independent of $x_1$ (i.e.\ $x^\star_1\cup x_2\cup x_3(\phi_3\land \phi_2)={\sf T}$ for any $x^\star_1\subseteq {\cal A}_1$). As ${{\cal A}_1}\indep_{\cal P}{\cal A}_2\mid {\cal A}_3$, $\IC_{\cal P}(x_1\cup x^\star_2\cup x_3)_{|1,3}= \IC_{\cal P}(x_1\cup x^\star_2\cup x_3)_{|1,3}$ for any $x^{\star}_2\subseteq {\cal A}_2$, we see that for any $x'_1\subseteq {\cal A}_1$,  $\IC_{\cal P}(x'_1\cup x^\star_2\cup x_3)_{|1}(\phi_1)={\sf T}$, which implies  $\phi_3\models_{\cal P} \phi_1$.
\end{proof}

\paragraph*{Splitting Operators}

We first recall the definitions on stratifiability. First, we denote, for a product lattice $\bigotimes_{i\in I} \Lc_i$, $x\in \bigotimes_{i\in I} \Lc_i$ and $j\in I$, $x_{|\leq j}=x_{|\{i\in I\mid i\leq j\}}$. An operator is \emph{stratifiable} (over $\bigotimes_{i\in I} \Lc_i$) iff for every $x^1,x^2\in \bigotimes_{i\in I} \Lc_i$ and every $j\in I$, if $x^1_{|\leq j}=x^2_{|\leq j}$ then $O(x)_{|\leq j}=O(y)_{|\leq j}$.

\begin{propositionAprep}
Let a $\leq_\otimes$-monotonic operator $O$  on the product lattice $\bigotimes_{i\in \{1,2,3\}} \Lc_i$ be given. Then $\Lc_1\indep_O \Lc_2\mid \Lc_3$ iff $O$ is stratifiable over $\Lc_1\otimes (\Lc_2\otimes \Lc_3)$ and $\Lc_2\otimes (\Lc_1\otimes \Lc_3)$.
\end{propositionAprep}
\begin{proof}
For the $\Rightarrow$-direction, suppose that $\Lc_1\indep_O \Lc_2\mid \Lc_3$. Suppose that $x^1,x^2\in \bigotimes_{i\in \{1,2,3\}} \Lc_i$ and that $x^1_1\otimes x^1_3=x^2_1\otimes x^1_3$. Then, as $\Lc_1\indep_O \Lc_2\mid \Lc_3$, $O(x^1)_{|\{1,3\}}=O^{1,3}(x^1_1\otimes x^1_3)=O^{1,3}(x^2_1\otimes x^2_3)=O(x^2)_{|\{1,3\}}$.

For the $\Leftarrow$-direction, suppose that $O$ is stratifiable over $\Lc_1\otimes (\Lc_2\otimes \Lc_3)$ and $\Lc_2\otimes (\Lc_1\otimes \Lc_3)$. Then we can define $O(x_1\otimes \otimes x_3)=O(x_1\otimes x_2\otimes x_3)_{|i,3}$ for any $x_2\in \Lc_2$ as $O(x_1\otimes x_2\otimes x_3)_{|i,3}=O(x_1\otimes x'_2\otimes x_3)_{|i,3}$ for any $x'_2\in \Lc_2$. 
\end{proof}

On the other hand, stratification does not, in general, imply conditional independence, as conditional independence requires symmetry:
\begin{example}
Consider ${\cal P}=\{q\leftarrow \sim r; r\leftarrow \sim s; s\leftarrow \sim p\}$. Then ${\cal P}$ can be stratified in the layers $\{p\},\{s,r\},\{q\}$ yet $\{q\}$ is not conditionally independent from any of the other atoms. 
\end{example}

\paragraph*{Decomposing Logic Programs}
We refer here to the relevant literature for background on treewidth-decompositions \cite{fichte2017answer}.

\begin{example}
To see that treewidth-decompositions do not always indicate a conditional independence, observe that  in Example \ref{example:graphical:criterion:too:naive} the tree-decomposition would suggest the conditional independence $\{a_1,b_1\}\indep_{\cal P} \{a_2,b_2\}\mid \{e\}$, which does not hold. 
\end{example}

\begin{example}
 Some decompositions are not visible using the purely syntactic approach from \cite{fichte2017answer}.
Let  ${\cal P}=\{p\leftarrow q,\sim q; q\leftarrow p,\sim p; q\leftarrow \sim r; p\leftarrow \sim s\}$, with the following dependency graph: 
\begin{center}
\begin{tikzpicture}[->, >=stealth', node distance=1.5cm, scale=0.5]
    \node[left of=e] (b1) {$r$};further ado
    \node (e) {$q$};
    \node[right of=e] (b2) {$p$};
    \node[right of=b2] (a1) {$s$};
    \path (b1) edge (e);
        \path (e) edge (b2);
        \path (b2) edge (e);
 \path (a1) edge (b2);
\end{tikzpicture}
\end{center}

The only treewidth decomposition is the following:

\begin{center}
\begin{tikzpicture}[-, >=stealth', node distance=2cm, scale=0.5]
    \node (pq) {$\{p,q\}$};
        \node[below left of=pq] (qr) {$\{q,r\}$};
    \node[below right of=pq] (ps) {$\{p,s\}$};
    \path (qr) edge (pq);
        \path (ps) edge (pq);
\end{tikzpicture}
\end{center}

However (since $p\leftarrow q,\sim q$ and $q\leftarrow p,\sim p$ are never applicable), it can be verified that $\{q,r\}\indep_{\cal P}\{p,s\}\mid \emptyset$.
\end{example}

\end{toappendix}

Other operator-based formalisms have been analysed in terms of treewidth decompositions \citep{fichte2022default,dvovrak2012towards}. A benefit of our operator-based approach is that all results are purely algebraic and therefore language-independent, which means that applications to specific formalisms are derived as straightforward corollaries. Furthermore, the results for AFT-based semantics, which subsume many KR-formalisms (an overview is provided by \cite{DBLP:journals/corr/abs-2305-10846}), are not restricted to the total stable fixpoints, but also apply to partial stable and well-founded semantics, in contrast to many studies on fixed-parameter tractability. 
 
Conditional independence has been investigated in several other logic-based frameworks, such as (iterated) belief revision \citep{lynn2022using,kern2022conditional}, conditional logics \citep{heyninck2022conditional} and formal argumentation \citep{rienstra2020independence,gaggl2021decomposition}. The benefit of our work is that the algebraic nature allows for the straightforward application to other  formalisms with a fixpoint semantics.

\section{Conclusion}
\label{sec:conclusion}
In this paper, the concept of conditional independence, well-known from probability theory, was formulated and studied for operators. This allows to use this concept to a wide variety of formalisms for knowledge representation that admit an operator-based characterisation. As a proof-of-concept, we have applied it to  normal logic programs. 

There exist several fruitful avenues for future work. Firstly, we want to investigate related notions of independence, such as context-specific independence \cite{boutilier1996context}.
A second avenue for future work is a more extensive application of the theory to concrete formalisms, both in breadth (by applying the theory to further formalisms) and in depth (e.g.\ by investigating more syntactic methods to identify conditional independencies, and by evaluating the computational gain experimentally).

\paragraph*{Acknowledgements}
I thank 
Hannes Stra{\ss} and Johannes P.\ Wallner  for interesting discussions on this topic. I thank the reviewers of KR 2024, to which a previous version of this paper was submitted, for their dilligent reviewing and constructive feedback, as well as the reviewers of this conference. 
This work was partially supported by the project LogicLM: Combining Logic Programs with Language Model with file number NGF.1609.241.010 of the research programme NGF AiNed XS Europa 2024-1 which is (partly) financed by the Dutch Research Council (NWO).

\bibliographystyle{kr} 
\bibliography{kr21,proposal}

\begin{thebibliography}{37}
\providecommand{\natexlab}[1]{#1}

\bibitem[{Bogaerts(2015)}]{phd/Bogaerts15}
Bogaerts, B. 2015.
\newblock \emph{Groundedness in logics with a fixpoint semantics}.
\newblock Ph.D. thesis, Informatics Section, Department of Computer Science,
  Faculty of Engineering Science.
\newblock Denecker, Marc (supervisor), Vennekens, Joost and Van den Bussche,
  Jan (cosupervisors).

\bibitem[{Bogaerts and Jakubowski(2021)}]{DBLP:journals/corr/abs-2109-08285}
Bogaerts, B.; and Jakubowski, M. 2021.
\newblock Fixpoint Semantics for Recursive {SHACL}.
\newblock In Formisano, A.; Liu, Y.~A.; Bogaerts, B.; Brik, A.; Dahl, V.;
  Dodaro, C.; Fodor, P.; Pozzato, G.~L.; Vennekens, J.; and Zhou, N., eds.,
  \emph{Proceedings 37th International Conference on Logic Programming
  (Technical Communications), {ICLP} Technical Communications 2021, Porto
  (virtual event), 20-27th September 2021}, volume 345 of \emph{{EPTCS}},
  41--47.

\bibitem[{Boutilier et~al.(1996)Boutilier, Friedman, Goldszmidt, and
  Koller}]{boutilier1996context}
Boutilier, C.; Friedman, N.; Goldszmidt, M.; and Koller, D. 1996.
\newblock Context-specific independence in Bayesian networks.
\newblock In \emph{Proc. 12th Conf. on Uncertainty in Artificial Intelligence
  (UAI'96)}, 115--123.

\bibitem[{Cousot and Cousot(1979)}]{cousot1979constructive}
Cousot, P.; and Cousot, R. 1979.
\newblock Constructive versions of Tarski’s fixed point theorems.
\newblock \emph{Pacific journal of Mathematics}, 82(1): 43--57.

\bibitem[{Darwiche(1997)}]{darwiche1997logical}
Darwiche, A. 1997.
\newblock A logical notion of conditional independence: properties and
  applications.
\newblock \emph{Artificial Intelligence}, 97(1-2): 45--82.

\bibitem[{Davey and Priestley(2002)}]{davey2002introduction}
Davey, B.~A.; and Priestley, H.~A. 2002.
\newblock \emph{Introduction to lattices and order}.
\newblock Cambridge university press.

\bibitem[{Denecker, Bruynooghe, and
  Vennekens(2012)}]{denecker2012approximation}
Denecker, M.; Bruynooghe, M.; and Vennekens, J. 2012.
\newblock Approximation fixpoint theory and the semantics of logic and answers
  set programs.
\newblock In \emph{Correct reasoning}, 178--194. Springer.

\bibitem[{Denecker, Marek, and
  Truszczy{\'n}ski(2000)}]{denecker2000approximations}
Denecker, M.; Marek, V.; and Truszczy{\'n}ski, M. 2000.
\newblock Approximations, stable operators, well-founded fixpoints and
  applications in nonmonotonic reasoning.
\newblock In \emph{Logic-based Artificial Intelligence}, volume 597 of
  \emph{The Springer International Series in Engineering and Computer Science},
  127--144. Springer.

\bibitem[{Denecker, Marek, and Truszczy{\'n}ski(2003)}]{denecker2003uniform}
Denecker, M.; Marek, V.; and Truszczy{\'n}ski, M. 2003.
\newblock Uniform semantic treatment of default and autoepistemic logics.
\newblock \emph{Artificial Intelligence}, 143(1): 79--122.

\bibitem[{Denecker, Marek, and Truszczynski(2002)}]{denecker2002ultimate}
Denecker, M.; Marek, V.~W.; and Truszczynski, M. 2002.
\newblock Ultimate approximations in nonmonotonic knowledge representation
  systems.
\newblock In \emph{Proceedings of the Eights International Conference on
  Principles of Knowledge Representation and Reasoning}, 177--190.

\bibitem[{Denecker, Marek, and Truszczy{\'n}ski(2004)}]{denecker2004ultimate}
Denecker, M.; Marek, V.~W.; and Truszczy{\'n}ski, M. 2004.
\newblock Ultimate approximation and its application in nonmonotonic knowledge
  representation systems.
\newblock \emph{Information and Computation}, 192(1): 84--121.

\bibitem[{Downey and Fellows(2013)}]{downey2013fundamentals}
Downey, R.~G.; and Fellows, M.~R. 2013.
\newblock \emph{Fundamentals of parameterized complexity}, volume~4.
\newblock Springer.

\bibitem[{Dvo{\v{r}}{\'a}k, Pichler, and Woltran(2012)}]{dvovrak2012towards}
Dvo{\v{r}}{\'a}k, W.; Pichler, R.; and Woltran, S. 2012.
\newblock Towards fixed-parameter tractable algorithms for abstract
  argumentation.
\newblock \emph{Artificial Intelligence}, 186: 1--37.

\bibitem[{Fichte et~al.(2017)Fichte, Hecher, Morak, and
  Woltran}]{fichte2017answer}
Fichte, J.~K.; Hecher, M.; Morak, M.; and Woltran, S. 2017.
\newblock Answer set solving with bounded treewidth revisited.
\newblock In \emph{Logic Programming and Nonmonotonic Reasoning: 14th
  International Conference, LPNMR 2017, Espoo, Finland, July 3-6, 2017,
  Proceedings 14}, 132--145. Springer.

\bibitem[{Fichte, Hecher, and Schindler(2022)}]{fichte2022default}
Fichte, J.~K.; Hecher, M.; and Schindler, I. 2022.
\newblock Default logic and bounded treewidth.
\newblock \emph{Information and Computation}, 283: 104675.

\bibitem[{Fitting(1991)}]{fitting1991bilattices}
Fitting, M. 1991.
\newblock Bilattices and the semantics of logic programming.
\newblock \emph{The Journal of Logic Programming}, 11(2): 91--116.

\bibitem[{Gaggl, Rudolph, and Strass(2021)}]{gaggl2021decomposition}
Gaggl, S.~A.; Rudolph, S.; and Strass, H. 2021.
\newblock On the decomposition of abstract dialectical frameworks and the
  complexity of naive-based semantics.
\newblock \emph{Journal of Artificial Intelligence Research}, 70: 1--64.

\bibitem[{Gottlob, Scarcello, and Sideri(2002)}]{gottlob2002fixed}
Gottlob, G.; Scarcello, F.; and Sideri, M. 2002.
\newblock Fixed-parameter complexity in AI and nonmonotonic reasoning.
\newblock \emph{Artificial Intelligence}, 138(1-2): 55--86.

\bibitem[{Heyninck(2024)}]{heyninck2024operatorbasedsemanticschoiceprograms}
Heyninck, J. 2024.
\newblock Operator-based semantics for choice programs: is choosing losing?
  (full version).

\bibitem[{Heyninck, Arieli, and
  Bogaerts(2022)}]{DBLP:journals/corr/abs-2211-17262}
Heyninck, J.; Arieli, O.; and Bogaerts, B. 2022.
\newblock Non-Deterministic Approximation Fixpoint Theory and Its Application
  in Disjunctive Logic Programming.
\newblock \emph{CoRR}, abs/2211.17262.

\bibitem[{Heyninck and
  Bogaerts(2023{\natexlab{a}})}]{DBLP:journals/tplp/HeyninckB23}
Heyninck, J.; and Bogaerts, B. 2023{\natexlab{a}}.
\newblock Non-deterministic Approximation Operators: Ultimate Operators,
  Semi-equilibrium Semantics, and Aggregates.
\newblock \emph{Theory Pract. Log. Program.}, 23(4): 632--647.

\bibitem[{Heyninck and
  Bogaerts(2023{\natexlab{b}})}]{DBLP:journals/corr/abs-2305-10846}
Heyninck, J.; and Bogaerts, B. 2023{\natexlab{b}}.
\newblock Non-deterministic approximation operators: ultimate operators,
  semi-equilibrium semantics and aggregates (full version).
\newblock \emph{CoRR}, abs/2305.10846.

\bibitem[{Heyninck, Kern-Isberner, and Meyer(2022)}]{heyninck2022conditional}
Heyninck, J.; Kern-Isberner, G.; and Meyer, T. 2022.
\newblock Conditional Syntax Splitting, Lexicographic Entailment and the
  Drowning Effect.

\bibitem[{Heyninck et~al.(2023)Heyninck, Kern{-}Isberner, Meyer, Haldimann, and
  Beierle}]{AAAI/HeyninckKM23}
Heyninck, J.; Kern{-}Isberner, G.; Meyer, T.~A.; Haldimann, J.; and Beierle, C.
  2023.
\newblock {Conditional syntax splitting for non-monotonic inference operators}.
\newblock In \emph{Proceedings of the 37th AAAI Conference on Artificial
  Intelligence (AAAI'23)}.

\bibitem[{Kern-Isberner, Heyninck, and Beierle(2022)}]{kern2022conditional}
Kern-Isberner, G.; Heyninck, J.; and Beierle, C. 2022.
\newblock Conditional independence for iterated belief revision.
\newblock In \emph{31st International Joint Conference on Artificial
  Intelligence}, 2690--2696. International Joint Conferences on Artificial
  Intelligence.

\bibitem[{Lang, Liberatore, and Marquis(2002)}]{lang2002conditional}
Lang, J.; Liberatore, P.; and Marquis, P. 2002.
\newblock Conditional independence in propositional logic.
\newblock \emph{Artificial Intelligence}, 141(1-2): 79--121.

\bibitem[{Liu and You(2022)}]{liu2022alternating}
Liu, F.; and You, J.-H. 2022.
\newblock Alternating fixpoint operator for hybrid MKNF knowledge bases as an
  approximator of AFT.
\newblock \emph{Theory and Practice of Logic Programming}, 22(2): 305--334.

\bibitem[{Lynn, Delgrande, and Peppas(2022)}]{lynn2022using}
Lynn, M.~J.; Delgrande, J.~P.; and Peppas, P. 2022.
\newblock Using conditional independence for belief revision.
\newblock In \emph{Proceedings of the AAAI Conference on Artificial
  Intelligence}, volume~36, 5809--5816.

\bibitem[{Makinson and van~der Torre(2003)}]{makinson2003input}
Makinson, D.; and van~der Torre, L. 2003.
\newblock What is input/output logic?
\newblock In \emph{Foundations of the Formal Sciences II: Applications of
  Mathematical Logic in Philosophy and Linguistics, Papers of a Conference held
  in Bonn, November 10--13, 2000}, 163--174. Springer.

\bibitem[{Pearl, Geiger, and Verma(1989)}]{pearl1989conditional}
Pearl, J.; Geiger, D.; and Verma, T. 1989.
\newblock Conditional independence and its representations.
\newblock \emph{Kybernetika}, 25(7): 33--44.

\bibitem[{Pelov, Denecker, and Bruynooghe(2007)}]{pelov2007well}
Pelov, N.; Denecker, M.; and Bruynooghe, M. 2007.
\newblock Well-founded and stable semantics of logic programs with aggregates.
\newblock \emph{Theory and Practice of Logic Programming}, 7(3): 301--353.

\bibitem[{Przymusinski(1990)}]{przymusinski1990well}
Przymusinski, T.~C. 1990.
\newblock The well-founded semantics coincides with the three-valued stable
  semantics.
\newblock \emph{Fundamenta Informaticae}, 13(4): 445--463.

\bibitem[{Rienstra et~al.(2020)Rienstra, Thimm, Kersting, and
  Shao}]{rienstra2020independence}
Rienstra, T.; Thimm, M.; Kersting, K.; and Shao, X. 2020.
\newblock Independence and D-separation in Abstract Argumentation.
\newblock In \emph{Proceedings of the International Conference on Principles of
  Knowledge Representation and Reasoning}, volume~17, 713--722.

\bibitem[{Strass(2013)}]{strass2013approximating}
Strass, H. 2013.
\newblock Approximating operators and semantics for abstract dialectical
  frameworks.
\newblock \emph{Artificial Intelligence}, 205: 39--70.

\bibitem[{Strass and Wallner(2015)}]{strass-wallner15complexity}
Strass, H.; and Wallner, J.~P. 2015.
\newblock Analyzing the Computational Complexity of Abstract Dialectical
  Frameworks via Approximation Fixpoint Theory.
\newblock \emph{Artificial Intelligence}, 226: 34--74.

\bibitem[{Van~Gelder, Ross, and Schlipf(1991)}]{van1991well}
Van~Gelder, A.; Ross, K.~A.; and Schlipf, J.~S. 1991.
\newblock The well-founded semantics for general logic programs.
\newblock \emph{Journal of the ACM}, 38(3): 619--649.

\bibitem[{Vennekens, Gilis, and Denecker(2006)}]{vennekens2006splitting}
Vennekens, J.; Gilis, D.; and Denecker, M. 2006.
\newblock Splitting an operator: Algebraic modularity results for logics with
  fixpoint semantics.
\newblock \emph{ACM Transactions on computational logic (TOCL)}, 7(4):
  765--797.

\end{thebibliography}

\end{document}